\documentclass[sigconf]{acmart}
\usepackage{stfloats}
\usepackage{url}
\usepackage{verbatim}
\usepackage{utfsym}
\usepackage{tcolorbox}
\usepackage{pifont}

\usepackage{amssymb}
\usepackage{mathrsfs}
\usepackage{multirow}
\usepackage{subcaption}
\usepackage{mathtools}
\usepackage{colortbl}
\usepackage{xcolor} 
\usepackage{enumerate} 
\usepackage{bm} 
\usepackage{graphicx}
\usepackage{booktabs}
\usepackage{setspace}
\usepackage{xcolor}
\usepackage{hyperref}
\hypersetup{
    colorlinks=true,
    linkcolor=blue,
    urlcolor=blue,
    citecolor=blue
}
\usepackage[linesnumbered,ruled,vlined]{algorithm2e}
\definecolor{pref}{HTML}{EFEFEF}
\definecolor{metrics}{HTML}{EFEFEF}
\definecolor{nd}{HTML}{80b918}
\definecolor{d}{HTML}{ef476f}
\newtheorem{assumption}{Assumption}[section]
\newtheorem{definition}{Definition}[section]
\allowdisplaybreaks
\AtBeginDocument{%
  }

\copyrightyear{2025}
\acmYear{2025}
\setcopyright{acmlicensed}\acmConference[KDD '25]{Proceedings of the 31st ACM SIGKDD Conference on Knowledge Discovery and Data Mining V.1}{August 3--7, 2025}{Toronto, ON, Canada}
\acmBooktitle{Proceedings of the 31st ACM SIGKDD Conference on Knowledge Discovery and Data Mining V.1 (KDD '25), August 3--7, 2025, Toronto, ON, Canada}
\acmDOI{10.1145/3690624.3709217}
\acmISBN{979-8-4007-1245-6/25/08}

\begin{document}
\ifodd 0
\newcommand{\rev}[1]{\textcolor{blue}{#1}}
\newcommand{\revw}[1]{\textcolor{red}{#1}}
\newcommand{\revg}[1]{\textcolor{cyan}{#1}}
\newcommand{\revh}[1]{#1}
\newcommand{\com}[1]{\textbf{\color{red} \left(Comment: #1\right) }}
\newcommand{\comg}[1]{\textbf{\color{blue} \left(COMMENT: #1\right)}}
\newcommand{\response}[1]{\textbf{\color{blue} \left(RESPONSE: #1\right)}}
\else
\newcommand{\rev}[1]{#1}
\newcommand{\revh}[1]{#1}
\newcommand{\revw}[1]{#1}
\newcommand{\com}[1]{}
\newcommand{\comg}[1]{}
\newcommand{\response}[1]{}
\fi
\title{PraFFL: A Preference-Aware Scheme in Fair Federated Learning}

\author{Rongguang Ye}
\email{yerg2023@mail.sustech.edu.cn}
\affiliation{%
  \institution{Southern University of Science and Technology}
  \city{Shenzhen}
  \country{China}
}
\author{Wei-Bin Kou}
\email{wbkou@connect.hku.hk}
\affiliation{%
  \institution{The University of
Hong Kong}
  \city{Hong Kong}
  \country{China}
}
\author{Ming Tang}
\authornote{Corresponding author.}
\email{tangm3@sustech.edu.cn}

\affiliation{%
  \institution{Southern University of Science and Technology}
  \city{Shenzhen}
  \country{China}
}


\begin{abstract}
Fairness in federated learning has emerged as a critical concern, aiming to develop an unbiased model among groups (e.g., male or female) of diverse sensitive features. However, there is a trade-off between model performance and fairness, i.e., improving model fairness will decrease model performance. Existing approaches have characterized such a trade-off by introducing hyperparameters to quantify client's preferences for model fairness and model performance. Nevertheless, these approaches are limited to scenarios where each client has only a single pre-defined preference, and fail to work in practical systems where each client generally has multiple preferences. To this end, we propose a \textbf{Pr}eference-\textbf{a}ware scheme in \textbf{F}air \textbf{F}ederated \textbf{L}earning (called PraFFL) to generate preference-specific models in real time. PraFFL can adaptively adjust the model based on each client's preferences to meet their needs. We theoretically prove that PraFFL can offer the optimal model tailored to an arbitrary preference of each client, and show its linear convergence. Experimental results show that our proposed PraFFL outperforms six fair federated learning algorithms in terms of the model's capability of adapting to clients' different preferences. Our implementation is available at \href{https://github.com/rG223/PraFFL}{\texttt{\textcolor[HTML]{ED028F}{https://github.com/rG223/PraFFL}}}.
\end{abstract}




\maketitle

\section{Introduction}
Federated learning (FL) is a distributed machine learning method that learns from data across multiple clients \cite{mcmahan2017communication,tan2022towards,wu2024breaking,wu2024heterogeneity,wu2024neulite}, where a global model is trained cooperatively among clients without sharing data. Since the global model learns from multiple data sources, it has a stronger generalization ability across multiple clients than separately training model for each client \cite{yuan2021we}. However, a key concern in FL is to ensure that the trained model is unbiased towards any particular groups (e.g., special races) of sensitive features  \cite{konevcny2016federated,hu2022fair,papadaki2022minimax}. This implies that a fair model needs to treat different groups equally, such that model's predictions are identical across different groups.

\begin{figure}[t]
    \centering
    \setlength{\abovecaptionskip}{0.1cm}
    \includegraphics[width=0.9\linewidth]{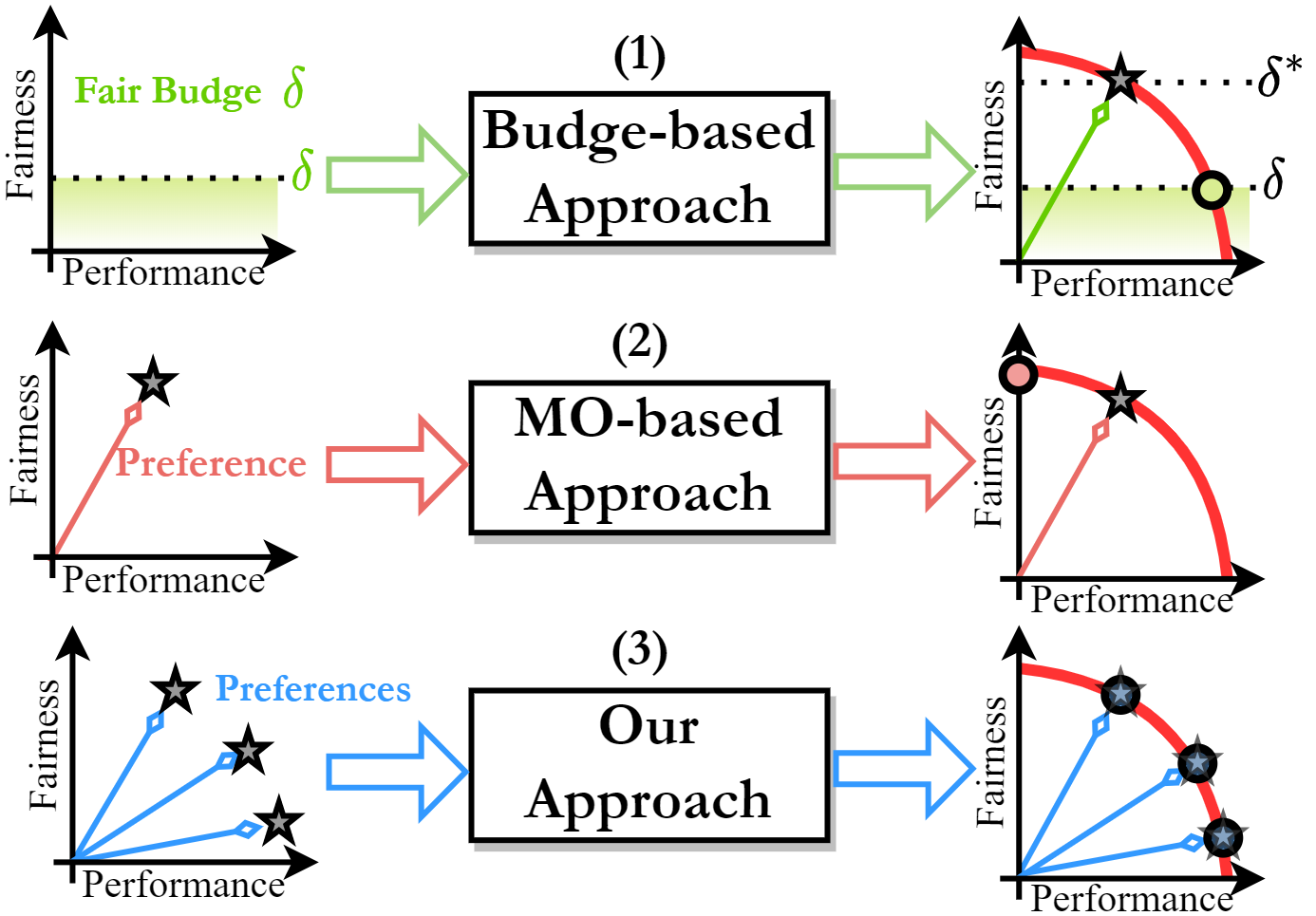}
    \caption{The model trained under different approaches. Gray stars represent the desired models, and colored circles represent the obtained models. The red curve indicates the optimal performance-fairness trade-offs, typically unknown in advance. $\delta^*$ is the fair budget that most closely aligns with the green preference arrow. For the first case, the desired model is obtained when $\delta=\delta^*$, where $\delta^*$ is the fair budget that most closely aligns with the green preference arrow. However, since the Pareto front is unknown, it is not possible to determine the exact value of $\delta^*$ that corresponds to the intersection of the preference direction and the Pareto front. Instead, $\delta$ must be repeatedly adjusted until it reaches $\delta^*$. For the second case, the obtained model deviates from the desired model due to the concave shape of the Pareto front (see Fig. \ref{fig:tch}). For the third case, our proposed approach not only achieves the desired models but also enables simultaneous handling of multiple preferences.} \label{motiv}
    \label{fig:enter-label}
\end{figure}

Fair FL was proposed to improve the fairness of the model \cite{zeng2021improving,abay2020mitigating,ezzeldin2023fairfed,pan2023fedmdfg}, where the fairness of the model is usually defined as the differences in predictions across different groups. \rev{Specifically, the smaller the prediction differences among different groups, the better the model's fairness.} Unfortunately, many studies \cite{gu2022privacy,yu2020fairness,du2021fairness} argue that improving model fairness and model performance in FL is a trade-off issue, i.e., model fairness and model performance cannot be improved simultaneously. Some existing works characterized the trade-off using hyperparameters. For example, budget-based approaches (e.g., LFT+Fedavg \cite{zeng2023federated} and Fairfed \cite{ezzeldin2023fairfed}) treat the fairness of the model as a constraint (see Fig. \ref{fig:enter-label}). Multi-objective-based (MO-based) approaches (e.g., FAIR-FATE \cite{salazar2023fair}) takes the linear combination of model performance and model fairness as the optimization objective. However, most of aforementioned works have two main limitations.
 First, their determined models may not satisfy a client's given preference over model fairness and model performance. Specifically, those budget-based approaches aim to match fairness but ignore performance. For example, Fig. \ref{fig:enter-label} depicts that the budget-based approaches need to repeatedly train the model after adjusting fair budget $\delta$ until the desired model is obtained. Those MO-based approaches (e.g., FAIR-FATE) can obtain the desired model only when the shape of the optimal performance-fairness curve (also called Pareto front) is convex \cite{boyd2004convex}, while such a convexity is not necessarily satisfied in FL scenario. Second, these approaches (e.g., \cite{abay2020mitigating,ezzeldin2023fairfed,salazar2023fair}) may not be applicable to the scenario where each client has multiple preferences. \rev{A practical example is to take the social platform TikTok as a client to meet preferences from billions of users \cite{jacques2023stopasianhate}. In particular,} users expect the model to recommend high-quality videos that align with their interests, i.e., these users prioritize model performance. In contrast, some other users seek unbiased recommendations towards/against any particular racial group, i.e., these users prioritize model fairness. For those aforementioned works (e.g., \cite{abay2020mitigating,ezzeldin2023fairfed,pan2023fedmdfg}), when each client has different requirements (i.e. preferences), their model needs to be retrained multiple times to adapt each requirement. Thus, their time complexity grows linearly with the number of preferences, as shown in Table \ref{mm}. Inspired by this, we aim to answer the following research questions:

\vspace{2pt} 
 \textbf{\textit{How can the model (i) effectively align with each client's specific preference and (ii) be adaptively adjusted in real time based on each client's diverse preferences?}}
\vspace{2pt} 

\begin{table}[t]
\renewcommand\arraystretch{1} 
\setlength{\abovecaptionskip}{0cm} 
\caption{
Time complexity comparison of multiple FFL methods with respect to the number of preferences $\kappa$.
} 
\label{mm}
\centering
\setlength{\tabcolsep}{2mm} 
\begin{tabular}{cccc}
\toprule
\multirow{2}{*}{\shortstack{Method}} & \multirow{2}{*}{\shortstack{Time\\Complexity}} & \multirow{2}{*}{\shortstack{Real-Time\\Adaption}} & \multirow{2}{*}{\shortstack{Preference\\Awareness}} \\ 
& & & \\ \midrule
LFT+Ensemble \cite{zeng2021improving}  & $\textcolor[HTML]{CB0000}{\mathcal{O}(\kappa)}$  & \textcolor[HTML]{CB0000}{\ding{55}}  & \textcolor[HTML]{CB0000}{\ding{55}} \\
LFT+Fedavg \cite{konecny2016federated} & $\textcolor[HTML]{CB0000}{\mathcal{O}(\kappa)}$  & \textcolor[HTML]{CB0000}{\ding{55}}  & \textcolor[HTML]{CB0000}{\ding{55}} \\
Agnosticfair \cite{du2021fairness}    & $\textcolor[HTML]{CB0000}{\mathcal{O}(\kappa)}$  & \textcolor[HTML]{CB0000}{\ding{55}}  & \textcolor[HTML]{CB0000}{\ding{55}} \\
FairFed \cite{ezzeldin2023fairfed}    & $\textcolor[HTML]{CB0000}{\mathcal{O}(\kappa)}$  & \textcolor[HTML]{CB0000}{\ding{55}}  & \textcolor[HTML]{CB0000}{\ding{55}} \\
FedFB \cite{zeng2021improving}        & $\textcolor[HTML]{CB0000}{\mathcal{O}(\kappa)}$  & \textcolor[HTML]{CB0000}{\ding{55}}  & \textcolor[HTML]{CB0000}{\ding{55}} \\
PraFFL (Ours)                         & $\textcolor[HTML]{009901}{\mathcal{O}(1)}$      & \textcolor[HTML]{009901}{\ding{51}}  & \textcolor[HTML]{009901}{\ding{51}} \\ 
\bottomrule
\end{tabular}
\end{table}

These problems involve three technical challenges: (I) establishing the connection between the two objectives (performance and fairness) and the client's preference; (II)  alleviating the impact of data heterogeneity; and (III) preventing the privacy leakage of client preferences. To this end, we propose a \textbf{Pr}eference-\textbf{A}ware scheme in \textbf{F}air \textbf{F}ederated \textbf{L}earning (PraFFL). Fig. \ref{motiv}(c) shows that our proposed PraFFL can accurately learn a mapping from the preference vector to a trade-off point between model performance and model fairness, enabling the model to link its fairness and performance with the corresponding preference vector. In PraFFL, we include a personalized FL method to alleviate the impact of data heterogeneity. Meanwhile, we introduce a hypernetwork that isolates each client's preference information, preventing it from being exchanged with other clients or the server, thereby ensuring client privacy. PraFFL can provide the preference-specific model to each client and achieve the time complexity of $O(1)$. In other words, once the model is trained by PraFFL, it allows the model to adapt to each client's preferences during the inference phase. In addition, PraFFL is theoretically proven to obtain the Pareto front within one training. 

The main contributions of this paper are highlighted as follows:
\vspace{-0.2cm}
\begin{itemize}
    \item We present a preference-aware learning framework, which enables the model to adapt to different preferences of each client.
    \item We propose a PraFFL scheme to further address the three technical challenges while allowing the model to adapt to arbitrary preferences of each client.
    \item We theoretically prove that PraFFL can linearly converge to optimal model that meets each client's preference, and can learn the Pareto front on each client's dataset.
    \item Numerical experiments validate that our proposed PraFFL outperforms other baselines on four datasets in terms of the model's capability in adapting to clients' different preferences. 
\end{itemize}

\vspace{-0.4cm}
\section{Related Work}\label{s2}
There are two types of fairness in fair FL: client-based fairness \cite{mohri2019agnostic,li2019fair,lyu2020collaborative,wang2021federated} and group-based fairness \cite{mohri2019agnostic,yue2023gifair,deng2020distributionally}. The purpose of client-based fairness is to reduce the variance of model performance across all clients while maintaining the performance of these models. Our work mainly discusses group-based fairness, in which each client has some sensitive features (e.g., gender and race). The goal of group-based fairness is to avoid discrimination against sensitive features  \cite{kamishima2012fairness,roh2020fairbatch,zhang2018mitigating}. Recent studies \cite{ezzeldin2023fairfed,du2021fairness,papadaki2024federated} investigated how to achieve group fairness in FL. For example, an adaptive FairBatch debiasing algorithm \cite{roh2020fairbatch} was proposed under the FL framework, where the optimization problem of each client is a bi-level optimization problem of model fairness and sampling probability for different groups. FedVal \cite{mehrabi2022towards} is another global reweighting strategy, where the client's aggregation weight depends on the fairness budget of the model in the validation set. However, previous studies have shown that improving model performance and model fairness is a trade-off problem \cite{papadaki2024federated,zeng2021improving,salazar2023fair,ezzeldin2023fairfed}. To control the trade-off between model performance and fairness, the most common approach is to treat fairness as a constraint in optimizing each client's model \cite{galvez2021enforcing,zhang2020fairfl,du2021fairness}. Specifically, FAIR-FATE \cite{salazar2023fair} introduces a weight coefficient to linearly combine fairness and model performance as the final optimization goal.
Furthermore, Fairfed \cite{ezzeldin2023fairfed} uses the FairBatch debiasing algorithm \cite{roh2020fairbatch} in the optimization stage to ensure fairness, and a reweighting method in the aggregation stage to control the performance-fairness trade-off. Most of the above methods can only obtain a model corresponding to one of the client's preferences. It is challenging for these methods to meet multiple user’s preferences.


\section{Problem Formulation}\label{s3}
\subsection{Multi-Objective Learning}
Let $\bm{b}\in \mathfrak{B}$ and $y \in \mathcal{Y}$ denote the non-sensitive features and labels of dataset $\mathcal{D}$, respectively. $a \in \mathcal{A}$ represents the sensitive feature where $a$ supports the inclusion of multi-dimensional sensitive information, for example, young man for both age and gender dimensions. The feature vector can be further expressed as $\bm{x}\triangleq(a,\bm{b}) \in \mathcal{X}$. According to \cite{zeng2021improving,chu2021fedfair}, we consider a binary classification problem (i.e., $y \in \mathcal{Y} \triangleq  \{0, 1\}$, where $y=1$ indicates positive class) and multi-group sensitive feature (i.e., $a \in \mathcal{A} \triangleq \{0,..., m\}$). Now, we introduce two objectives in fair FL, i.e., model performance and model fairness.

\textbf{Model Performance.} Let $\bm{\theta}: \mathcal{X} \rightarrow \mathcal{Y}$ be a classifier. Usually, the cross-entropy loss function is used to quantify the model performance in the classification problem, and it is defined as follow:
\begin{equation}\label{celoss}
         \mathcal{L}_{1}(\bm{x},y,\bm{\theta})\!=\!-[y\mathrm{log}(\bm{\theta}(\bm{x}))+(1\!-y)\mathrm{log}(1-\bm{\theta}(\bm{x}))],
\end{equation}
where $\bm{\theta}(\bm{x})$ represents probability of predicting a data sampled with feature $\boldsymbol{x}$ belongs to the positive class. 

\textbf{Model Fairness.} A widely adopted criterion for assessing classifier fairness is \textit{DP disparity} \cite{feldman2015certifying}. Let $\hat{y}$ represent the the label predicted by the classifier based on feature $\boldsymbol{x}$. DP disparity quantifies the maximum difference in expectation of the classifier $\bm{\theta}$ that predicts a data sample to be in positive class across different sensitive groups
\begin{equation}\label{dp}
   \underset{j\in\mathcal{A}}{\mathrm{max}}|\mathbb{E}_{(\bm{x},y)\sim \mathcal{D}_k}[\hat{y}=1{\mid a=j}]-\mathbb{E}_{(\bm{x},y)\sim \mathcal{D}_k}[\hat{y}=1]|.
\end{equation}
  Eq. (\ref{dp}) is not suitable for optimization (in gradient-based training) due to its discontinuity. Instead, model fairness is usually characterized by the following loss function \cite{zeng2021improving}
\begin{equation} \label{fairloss}
\vspace{-0.1cm}
    \mathcal{L}_{2}(\bm{x},y, \bm{\theta})=[(a -\overline{a})(\bm{\theta}(\bm{x}) -\overline{\bm{\theta}}(\bm{x}))]^2,
\end{equation}
where $\overline{a}$ and $\overline{\bm{\theta}}(x)$ represent the average values of $a$ and $\bm{\theta}(x)$ over $\mathcal{D}$, respectively. $\mathcal{L}_{2}$ quantifies the correlation between sensitive features and the model predictions. A higher value of $\mathcal{L}_{2}$ suggests that sensitive features are more strongly correlated with the model's predicted probabilities, indicating a lower fairness.

 Then, we convert the fair FL into a multi-objective problem to solve, which can be expressed as follows:
\begin{equation}\label{moo_problem}
    \min_{\boldsymbol{\theta}\in \Theta }\boldsymbol{\mathcal{L}}(\bm{x},y, \bm{\theta})=(\mathcal{L}_1(\bm{x},y, \bm{\theta}), \mathcal{L}_{2}(\bm{x},y, \bm{\theta})),
\end{equation}
where $\bm{\mathcal{L}}(\bm{x},y, \bm{\theta}): \mathcal{X} \times \mathcal{Y} \times \Theta \rightarrow \mathbb{R}^{2}_{\geq 0}$ is a two-dimensional objective vector. \textit{For presentation simplicity, in the rest of this paper, we use the terms “solution” and “trained model” interchangeably}. 

Model performance and model fairness are conflicting, and we need to make trade-offs between these two objectives \cite{papadaki2024federated,zeng2021improving}. For
ease of notation, let us denote $j$-th objective as $\mathcal{O}_{j}(\bm{\theta})=\mathcal{L}_{j}(\bm{x},y,\bm{\theta})$ for all $j \in \{1, 2\}$. We give the following definitions of fair FL in the context of multi-objective optimization \cite{coello2007evolutionary}.
\begin{definition}[Pareto Dominance] \label{dominated}
 Let $\bm{\theta^a}$, $\bm{\theta^b}\in \Theta$, $\bm{\theta^a}$ is said to dominate $\bm{\theta^b}$, denoted as $\bm{\theta^a}\prec \bm{\theta^b}$, if and only if $\mathcal{O}_i(\bm{\theta^a})\leq \mathcal{O}_i(\bm{\theta^b})$ for all $ i\in\{1, 2\}$ and $\mathcal{O}_j(\bm{\theta^a})<\mathcal{O}_j(\bm{\theta^b})$ for some $j\in\{1, 2\}$. In addition, $\bm{\theta}^{a}$ is said to strictly dominate $\bm{\theta}^{b}$, denoted as $\bm{\theta}^{a} \prec_{strict} \bm{\theta}^{b}$, if and only if $\mathcal{O}_{i}(\bm{\theta}^{a}) < \mathcal{O}_{i}(\bm{\theta}^{b})$ for all $i \in \{1, 2\}.$
\end{definition}

\begin{definition}[Pareto Optimality]\label{paretoop}
     $\bm{\theta}^*\in \Theta$ is Pareto optimal if there does not exist $\hat{\bm{\theta}}\in \Theta$ such that $\hat{\bm{\theta}}\prec \bm{\theta}^*$. In addition, $\bm{\theta}^{'}$ is weakly Pareto optimal if there is no $\bm{\theta} \in \Theta$ such that $\bm{\theta} \prec_{strict} \bm{\theta}^{'}$.
\end{definition}

\begin{definition}[Pareto Set/Front]
    The set of all Pareto optimal solutions $\mathcal{P}^{*} \subseteq \Theta$ is called the Pareto set, and its image in the objective space $\mathcal{PF}^{*} = \{\bm{\mathcal{O}}(\bm{\theta}) \mid \bm{\theta} \in \mathcal{P}^{*}\}$ is called the Pareto front. 
\end{definition}

To evaluate the quality of a solution set, a hypervolume (HV) indicator \cite{fonseca2006improved} is commonly used.

\begin{definition}[Hypervolume] \label{hv}
    Let $\hat{\Theta}=\{\bm{\theta}^{(1)}, ..., \bm{\theta}^{(n)}\}$ be a set of objective vectors given solution $\bm{\theta}^{(i)}$, $i \in$ [$n$]. The $hypervolume$ indicator $\mathcal{H}_{\bm{r}}(\hat{\Theta})$ is the two-dimensional Lebesgue measure of the region dominated by $\mathcal{P}^{*}(\hat{\Theta})$ and bounded by a pre-defined reference point $\bm{r}\in\mathbb{R}^2.$
\end{definition}
Fig. \ref{fig:hv} shows the calculation of HV for a set of five solutions, where the gray area represents the HV value. We aim to make the model adapt to client preferences, and to maximize the HV of the Pareto front in the objective space.

\subsection{Preference-Aware Learning Framework}
    Inspired by the scalarization method in multi-objective optimization, we introduce preference vector $\bm{\lambda} \in \Lambda \in \mathbb{R}^{2}_{_{> 0}} (\sum_{i=1}^{2}\lambda_{i}=1)$ and incorporate it into the original multi-objective optimization problem (Eq. (\ref{moo_problem})) to characterize clients' preference information. In particular, we consider a generalized setting \cite{li2021ditto} where each client $k$ aims to obtain its own model $\bm{\theta}_k$. This setting can generalize the conventional FL case where clients aim to learn an identical global model. To enable the model to adapt to each client's individual preference, we concatenate ($\oplus$) the input $\bm{x}$ with the preference vector $\bm{\lambda}$ as joint features for training each client's local model $\bm{\theta}_{k}: \mathcal{X} \oplus \Lambda \rightarrow \mathcal{Y}$. By considering the preference vector in FL system, Eq. (\ref{moo_problem}) can be transformed to
\begin{equation}\label{basic}
\min_{(\bm{\theta}_{1},\dots,\bm{\theta}_{K}) \in \Theta_{K}} \frac{1}{K} \sum_{k\in [K]} \mathbb{E}_{\bm{\lambda}\sim \Lambda}  \mathbb{E}_{(\bm{x},y)\sim \mathcal{D}_k} \bm{\lambda} \cdot \bm{\mathcal{L}}\big(\bm{x} \oplus \bm{\lambda},y, \boldsymbol{\theta}_{k}\big),
\end{equation}  
\begin{figure}[t]
    \centering
    \includegraphics[width=0.6\linewidth]{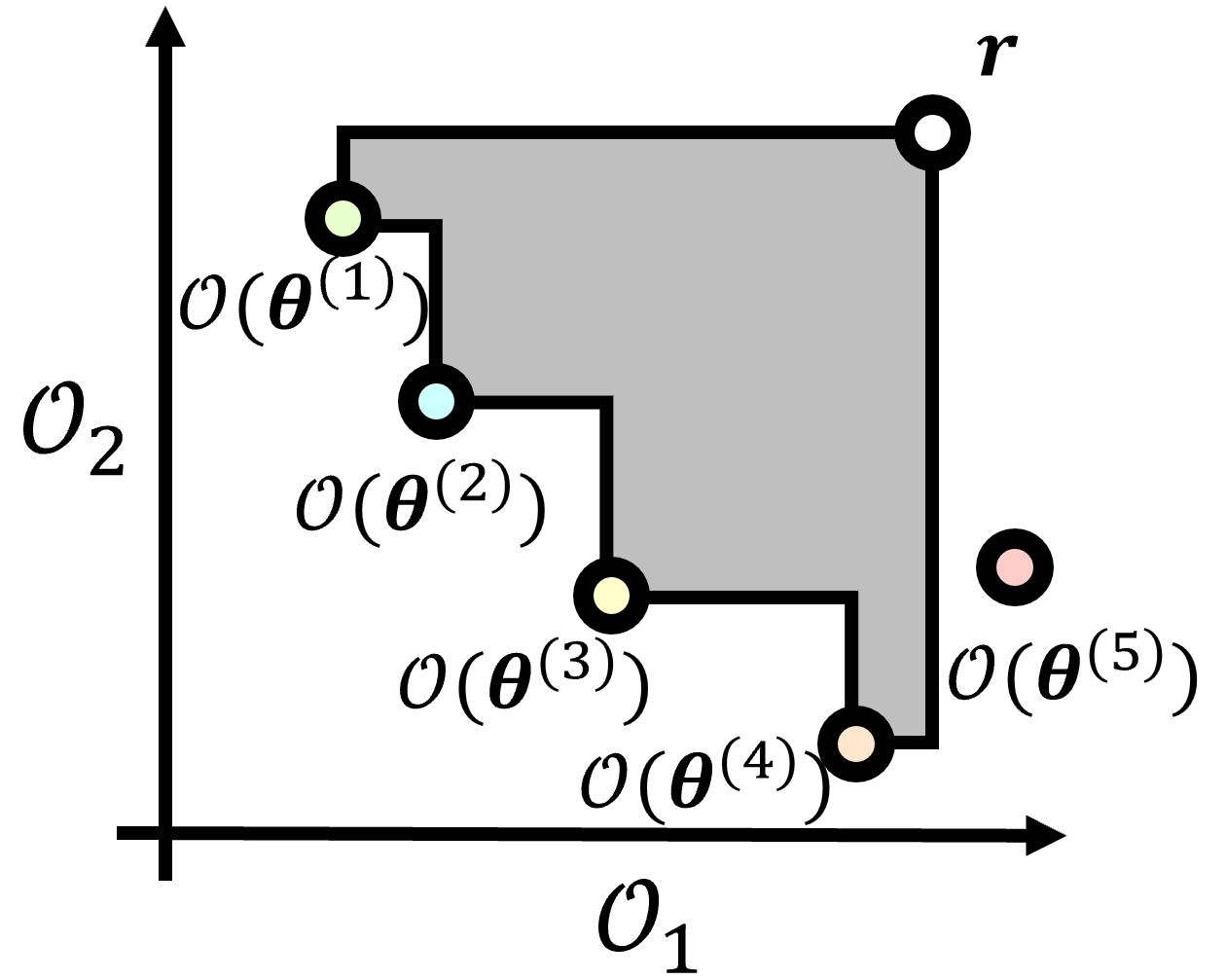}
    \caption{Diagram of hypervolume.}
    \label{fig:hv}
\end{figure}
where $\Theta_{K}$ is the feasible space of $K$ models. $\mathcal{D}_{k}$ represents the dataset of client $k$. The system aims to provide models tailored to any preference. Hence, the system must be trained on the expectation of \textit{preference distribution} $\Lambda$ to encompass all potential preferences. In practice, the distribution $\Lambda$ can be set to a uniform distribution when there is no prior knowledge of clients. Once the model $\bm{\theta}_{k}$ ($k \in [K]$) is optimized, the prediction results of each client's dataset will be determined by each client's given preference in the inference phase. The Pareto front can be generated by examining the models of all possible preference vectors in the inference phase (i.e., one preference generates one Pareto optimal solution). 

Although our proposed preference-aware learning framework (Eq. (\ref{basic})) takes client preferences into account, there are three difficulties in obtaining high-quality Pareto front by Eq. (\ref{basic}):
\begin{enumerate}[(I)]
    \item We have no prior knowledge of the Pareto front shape. According to \cite{boyd2004convex}, a weight-sum function in Eq. (\ref{basic}) only works when the Pareto front is convex. Eq. (\ref{basic}) cannot guarantee the Pareto optimal solution is in the direction of preference vector when the Pareto front is concave (illustrated in Fig. \ref{fig:tch}(a)).
    \item The data distribution among different clients is heterogeneous, that is, their Pareto fronts are also different (depending on each client's data distribution).
    \item Optimizing the system via Eq. (\ref{basic}) may leak preference information of clients (e.g., data reconstruction attack \cite{singhal2021federated}) because preference vectors are used as input data.
\end{enumerate}

\section{The Proposed Method: P\MakeLowercase{ra}FFL} \label{modeling}
\begin{figure}[t]
    \centering
    \includegraphics[width=0.95\linewidth]
    {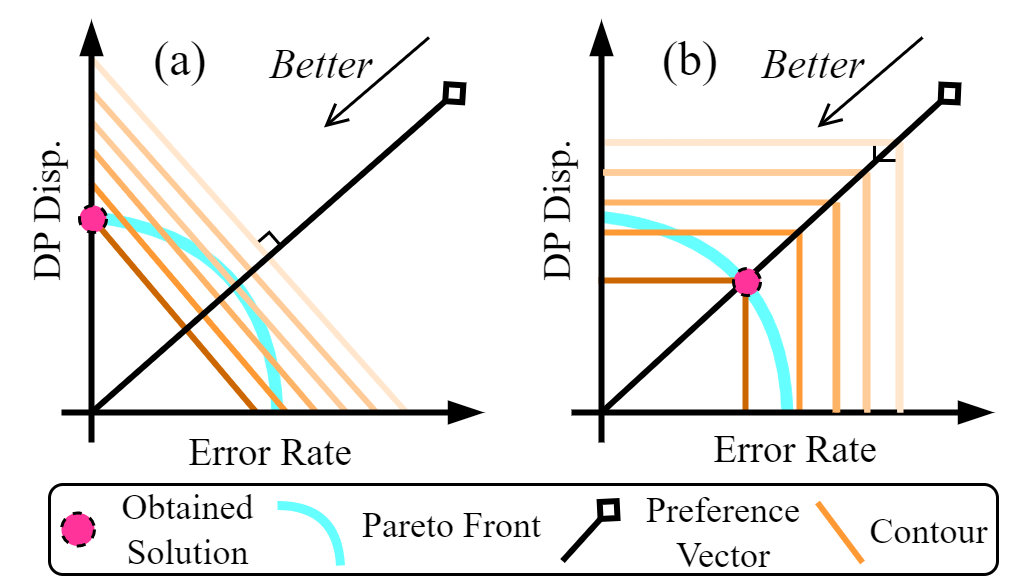}
    \caption{The solution obtained by different optimization functions. (a): The weighted-sum function (Eq. (\ref{basic})). (b): The weighted Tchebycheff function (Eq. (\ref{tch_fl})).}
  \label{fig:tch}
\end{figure}
 In this section, we propose a preference-aware scheme in fair FL (PraFFL) to further address three difficulties in solving Eq. (\ref{basic}).
Section \ref{singlemoo} presents a preference alignment approach to address difficulty (I). Moreover, Section \ref{objective} incorporates client preferences into personalized FL to address difficulty (II). Furthermore, we propose to use a hypernetwork to consider preferences in Section \ref{sechypernet}, which protects the clients' preference privacy (addressing difficulty (III)). In addition, Section \ref{solver} summarizes the algorithm details of PraFFL scheme. 

\subsection{Alignment of Model with Preference}\label{singlemoo}
It is difficult to provide a model whose performance and fairness fall in the direction of the client's corresponding preference vector. To obtain a Pareto optimal solution in the direction of any given preference vector, we introduce a weighted Tchebycheff function \cite{miettinen1999nonlinear} as the optimization function for each client as follows:
\begin{equation} \label{tch}
     \min_{\bm{\theta}_{k}}g_{\mathrm{t}}(\bm{x} \oplus \bm{\lambda},y, \boldsymbol{\theta}_{k} \mid \boldsymbol{\lambda})=\min_{\bm{\theta}_{k}}  \max_{i \in \{1, 2\}}  \left\{\frac{\mathcal{L}_{i}\big(\bm{x} \oplus \bm{\lambda}, y, \boldsymbol{\theta}_{k}\big)}{\lambda_{i}}\right\}. 
\end{equation}
We can obtain the following promising property of problem (\ref{tch}).

\begin{lemma}[Preference Alignment \cite{miettinen1999nonlinear}] \label{them1}
Given a preference vector $\bm{\lambda} > 0$, a solution $\bm{\theta}_{k} \in \Theta$ is weakly Pareto optimal to problem \ref{moo_problem} if and only if $\bm{\theta}_{k}$ is an optimal solution to problem (\ref{tch}).
\end{lemma}

Lemma \ref{them1} shows that given a preference vector $\bm{\lambda}$, our method can obtain a weakly Pareto optimal solution for each client if problem (\ref{tch}) is minimal. Then, substituting Eq. (\ref{tch}) into problem (\ref{basic}), the optimization problem in our system can be represented as
\begin{equation}\label{tch_fl}
    \min_{(\bm{\theta}_{1},\dots,\bm{\theta}_{K}) \in \Theta_{K}} \frac{1}{K} \sum_{k\in [K]}  \mathbb{E}_{\bm{\lambda} \sim \Lambda} \mathbb{E}_{(\bm{x},y)\sim \mathcal{D}_k} g_{\mathrm{t}}(\bm{x} \oplus \bm{\lambda},y, \bm{\theta}_{k} \mid \bm{\lambda}).
\end{equation}
\rev{Eq. (\ref{tch_fl}) can optimize each model towards the direction of the preference vector.} Fig. \ref{fig:tch}(a) shows the result obtained by weight-sum function (Eq. (\ref{basic})), which cannot accurately obtain the model on the preference vector. In Fig. \ref{fig:tch}(b), the weight Tchebycheff optimization function we introduced can accurately obtain the model on the preference vector if Eq. (\ref{tch_fl}) is minimum. In this way, we can align the model to client's given preference vector. However, due to the data heterogeneity of clients, problem ({\ref{tch_fl}}) cannot reach the minimum value, resulting in the failure to meet the condition of Lemma \ref{them1}. In the following subsections, we introduce personalized FL based on hypernetwork mapping, which allows each client to obtain its own optimal model without revealing the privacy of client preferences. 

\begin{figure}[t]
    \centering
    \includegraphics[width=\linewidth]{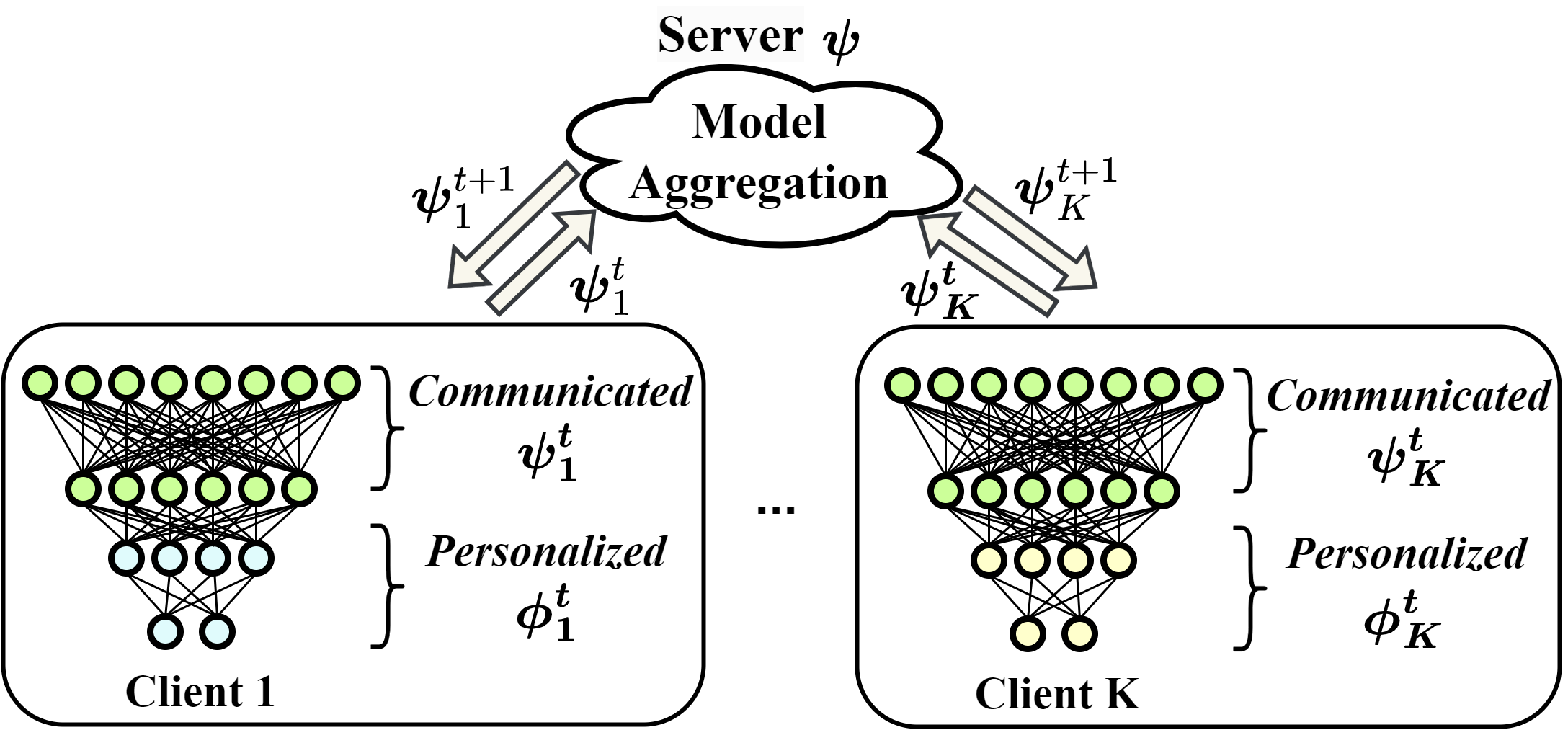}
    \caption{Personalized federated learning framework. 
    }
    \label{fig:pfl}
\end{figure}
\subsection{Personalized FL Based on Preferences}\label{objective}
To address difficulty (II) in Section \ref{s3}, we present a personalized FL framework (see Fig. \ref{fig:pfl}). In this framework, each client's model is divided into two parts (communicated model and personalized model). The communicated model ${\boldsymbol{\psi}}: \mathbb{R}^{|\mathcal{X}|} \rightarrow \mathbb{R}^{d}$ is the part of the model that is shared between clients. $d$ is the dimension of the data after passing through the communicated model. Each client also has a personalized model ${\boldsymbol{\phi}_{k}}: \mathbb{R}^{d} \rightarrow \mathbb{R}^{|\mathcal{Y}|}$. The model of client $k$ is expressed as $\bm{\theta}_k=({\boldsymbol{\psi}}, {\boldsymbol{\phi}}_{k})$. The inclusion of the personalized model can solve difficulty (II) because each client's personalized model is optimized according to its corresponding data distribution. Thus, by taking into account both the communicated and personalized models, the optimization goal of PraFFL can be transformed into
\begin{subequations}\label{bi}
\begin{align} 
& \!\underset{(\boldsymbol{\phi}_{1}, \dots, \boldsymbol{\phi}_{K}) \in \Phi_{K}}{\min} \frac{1}{K}\! \sum_{k \in [K]}\mathbb{E}_{\bm{\lambda} \sim \Lambda} \mathbb{E}_{(\bm{x},y)\sim \mathcal{D}_k} g_{\mathrm{t}} (\bm{x} \oplus \bm{\lambda},y, \bm{\theta}_{k} \mid \bm{\lambda}),  \\
& s.t. \ \boldsymbol{\psi} = \underset{\boldsymbol{\psi} \in \Psi}{\arg \min }  \ \frac{1}{K}\sum_{k \in [K]} \mathbb{E}_{\bm{\lambda}\sim \Lambda}\mathbb{E}_{(\bm{x},y)\sim \mathcal{D}_k}\mathcal{L}_{1}\left(\bm{x} \oplus \bm{\lambda}, y, \bm{\theta}_k\right). 
\end{align}
\end{subequations}

Our idea is to ensure the performance of the communicated model (Eq. (\ref{bi}b)), and then optimize the personalized model according to the given preferences (Eq. (\ref{bi}a)). As will be seen in later sections, we consider an iterative optimization process. When Eq. (\ref{bi}b) is being optimized, clients collaboratively optimize a global communicated model $\boldsymbol{\psi}$ to minimize cross-entropy loss (performance) over their own datasets. After Eq. (\ref{bi}b) has been optimized, the system optimizes the personalized models of clients, considering the expected model performance and fairness over clients' preferences. Note that the optimization of communicated model takes into account the preference vector as input. These preference vectors are easily leaked through data reconstruction (difficulty (III)) \cite{jin2021cafe}.

\subsection{Coupled Preferences with Hypernetwork \label{sechypernet}}
To further solve difficulty (III), we introduce a \textit{hypernetwork} to prevent the leakage of the client's preference information. Hypernetwork is deployed on every client. Let ${\bm{\beta}_{k}}: \mathbb{R}^{m} \rightarrow \mathbb{R}^{|\boldsymbol{\phi}_{k}|}$ be a hypernetwork, which aims to learn a mapping from the preference vector to the parameters of the personalized model. Given a preference vector $\bm{\lambda}$, the personalized model ${{\boldsymbol{\phi}}_k}$ of client $k$ is determined as follows:
\begin{equation} \label{hyper}
    \boldsymbol{\phi}_{k} = {\bm{\beta}_k}(\bm{\lambda}).
\end{equation}
\rev{Based on Eq. (\ref{hyper}), client's preference $\bm{\lambda}$ are represented as the personalized model via a hypernetwork rather than as part of the joint features (as mentioned in Eq. (\ref{basic})).} Each client $k$ can directly obtain its personalized model by specifying vector $\bm{\lambda}$ in Eq. (\ref{hyper}). \rev{As shown in Fig. \ref{fig:client}, if the hypernetwork can be well-trained, then it can generate the personalized model that satisfies the corresponding preference in the inference phase.} The model of client $k$ is further expressed as $\bm{\theta}_k(\bm{\lambda})=(\boldsymbol{\psi},\bm{\beta}_k(\bm{\lambda}))$. The optimization problem of our system is transformed into the following
\begin{subequations}
\begin{align}
& \underset{(\bm{\beta}_{1},\dots,\bm{\beta}_{K}) \in \mathcal{B}_{K}}{\min} \ \frac{1}{K} \sum_{k \in [K]} \mathbb{E}_{\bm{\lambda} \sim \Lambda} \mathbb{E}_{(\bm{x},y)\sim \mathcal{D}_k} g_{\mathrm{t}} (\bm{x},y, {\bm{\theta}_k}(\bm{\lambda}) \mid \bm{\lambda}),  \\
& s.t. \ \boldsymbol{\psi} = \underset{\boldsymbol{\psi} \in {\Psi}}{\arg \min } \ \frac{1}{K} \sum_{k \in [K]} \mathbb{E}_{\bm{\lambda}\sim \Lambda}\mathbb{E}_{(\bm{x},y)\sim{\mathcal{D}_k}}\mathcal{L}_{1}(\bm{x},y, \bm{\theta}_k({\bm{\lambda}}))). 
\end{align}\label{bii}
\end{subequations}

Eq. (\ref{bii}a) aims to learn a mapping relation from preference vector distribution $\Lambda$ to a set of personalized models. Due to the client's preferences being learned in the hypernetwork, each client's preference information is isolated from other clients. Therefore, hypernetwork can learn each client's preferences while protecting the client's preference information  (addressing difficulty (III)). So far, the three difficulties mentioned in Section \ref{s3} have been solved. 

\begin{figure}[t]
    \centering
    \includegraphics[width=\linewidth]{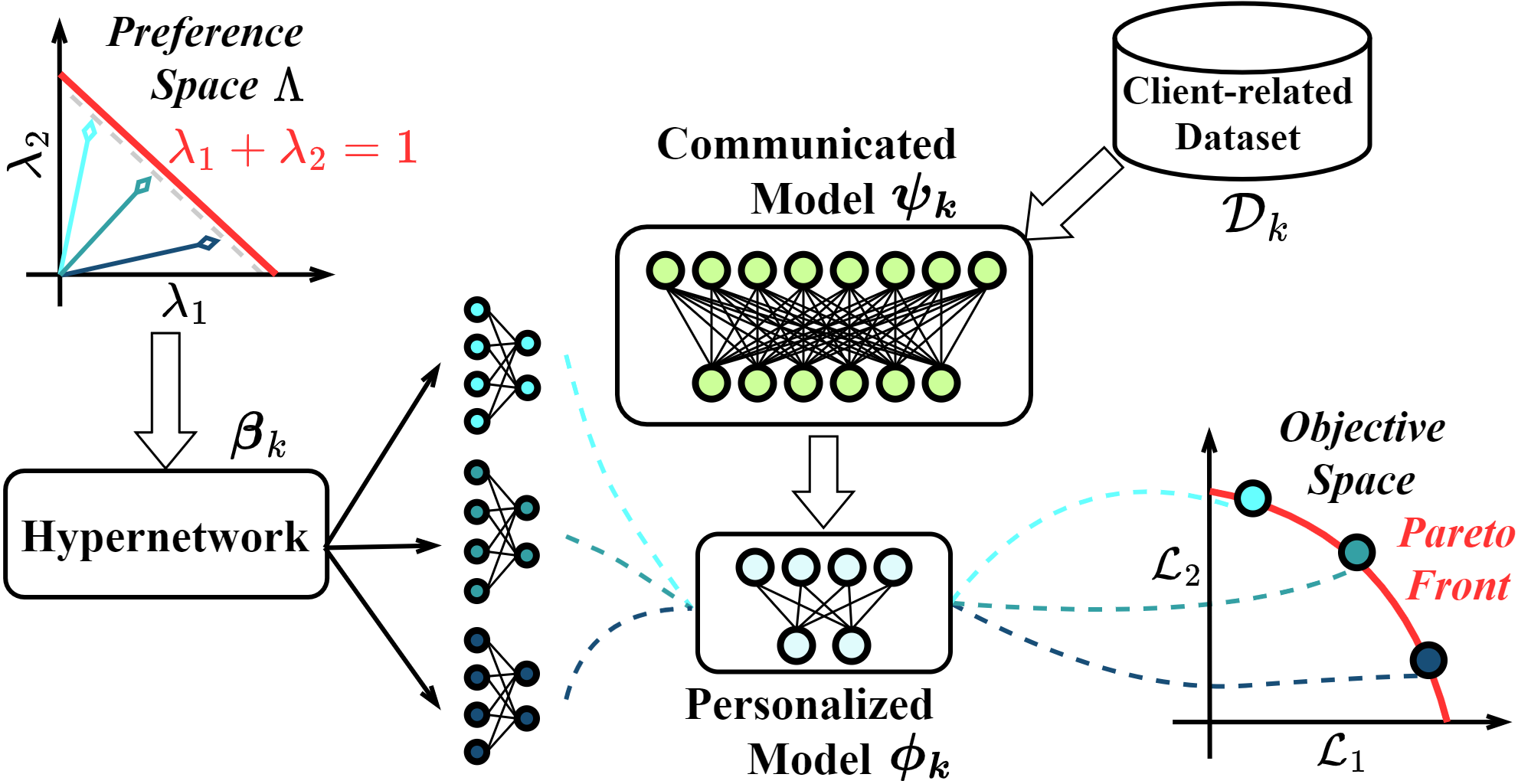}
    \caption{The illustration of hypernetwork inference for client $k$. Hypernetwork ${\bm{\beta}_{k}}$ returns the corresponding personalized models to client $k$ based on the preference vectors.}
    \label{fig:client}
\end{figure}

\subsection{PraFFL Algorithm}\label{solver}
Now, we propose PraFFL algorithm to solve problem (\ref{bii}), which can generate a preference-specific model for each client while addressing difficulties (I)--(III). 
Specifically, there are $T$ rounds in total. In each round, a proportion of $p \in (0, 1]$ clients is selected for training, where the set of clients is denoted as $\mathcal{S}_t$. Then, the communicated model and the hypernetwork are updated sequentially at each client. Each client $k$ first downloads the global communicated model $\boldsymbol{\psi}$ as its local communicated model $\boldsymbol{\psi}_k$. 

To stabilize the optimization of the communicated model for solving problem (\ref{bii}b), we need to froze the personalized model when the communicated model is being updated. In particular, personalized model of client $k$ is frozen to $\bm{\phi}_k = \bm{\beta}(\check{\bm{\lambda}})$, where $\check{\bm{\lambda}}$ is chosen to be $(\frac{1}{2}, \frac{1}{2})$ so that the personalized model is unbiased towards model performance and fairness when optimizing the local communicated model. Client $k$ performs $\tau_{c}$ times of gradient descent algorithm for updating the communicated model at $t$-th round
\begin{equation}\label{communicated}
    \boldsymbol{\psi}^{t,s+1}_{k} =  \boldsymbol{\psi}^{t,s}_{k}- \eta \mathbb{E}_{(\bm{x},y)\sim \mathcal{D}_k}\nabla_{\boldsymbol{\psi}_{k}} \mathcal{L}_{1}(\bm{x},y, {\bm{\theta}^{t,s}_{k}}(\check{\bm{\lambda}})),
\end{equation}
where $s=1,...,\tau_{c}$ and $\eta$ is the learning rate.

After updating the communicated model, each client then optimizes the hypernetwork. When optimizing problem (\ref{bii}a), client $k$ freezes the communicated model as $\boldsymbol{\psi}_{k}^{t,\tau_{c}+1}$. The optimization of the hypernetwork for client $k$ can be expressed as follows:

\begin{equation} \label{expect}
\underset{\boldsymbol{\beta}_{k}}{\min} \ \mathbb{E}_{\boldsymbol{\lambda}\sim\Lambda} \mathbb{E}_{(\bm{x},y)\sim \mathcal{D}_k} g_{\mathrm{t}}(\bm{x},y, {\bm{\theta}_{k}}(\boldsymbol{\lambda})).
\end{equation}
Eq. (\ref{expect}) is difficult to optimize because there are infinite possible values for $\bm{\lambda}$. We use the Monte Carlo sampling method in stochastic optimization to optimize Eq. (\ref{expect}). Thus, client $k$ performs $\tau_p$ 
\begin{equation}\label{peason}
\! \! \!  \boldsymbol{\beta}^{t,j+1}_{k} \! =\boldsymbol{\beta}^{t,j}_k-\eta \mathbb{E}_{(\bm{x},y)\sim \mathcal{D}_k} \frac{1}{n_{\bm{\lambda}}}\sum_{i=1}^{n_{\bm{\lambda}}}\nabla_{\boldsymbol{\bm{\beta}_{k}}}g_{\mathrm{t}}(\bm{x}, y, {\bm{\theta}^{t,j}_{k}}(\boldsymbol{\lambda}^{(i)})),\!
\end{equation}
where $j=1,\dots,\tau_{p}$, $n_{\bm{\lambda}}$ is the batch size for sampling preference vectors and $\bm{\lambda}^{(i)}$ is sampled from preference vector distribution $\Lambda$. In our work, we set $\Lambda$ to be a uniform distribution.

Once the communicated model and personalized model of participated clients are updated at $t$-th round, the participated clients send their communicated models ${\boldsymbol{\psi}_{k}} (k\in [K]$) to the server. Afterwards, the server updates the communicated model with average aggregation as follows:
\begin{equation}\label{agg}
  \boldsymbol{\psi}^{t+1,1}_{k}=\frac{1}{|\mathcal{S}_{t}|}\sum_{k \in \mathcal{S}_{t}}\boldsymbol{\psi}^{t,\tau_c+1}_{k}.  
\end{equation}

After the model aggregation in Eq. (\ref{agg}), the communicated models at $(t+1)$-th round are the same across all clients. In contrast, the personalized model is customized based on the specific data distribution of each client. As shown in Fig. \ref{fig:client}, each client can choose arbitrary preference vectors, and the hypernetwork then provides the corresponding personalized models to the client. Then, each client can combine the personalized model with the communicated model, and the combined model can produce the desired results considering the client's preference over model performance and model fairness.

\section{ALGORITHMIC ANALYSIS}\label{analyse}
In this section, we analyze performance guarantee of the trained model and the convergence rate of PraFFL. The proofs of this section are provided in the Appendix.

\subsection{Learning Capabilities}
Based on the properties of the weighted Tchebycheff function \cite{miettinen1999nonlinear}, we determine the following three main results:

\begin{itemize}
\item Existence of the Pareto optimal model;
\item Condition of Pareto model being Pareto optimal;
\item  The ability to find the Pareto front.
\end{itemize}

We analyze the model of an arbitrary client $k$. Given a preference vector $\bm{\lambda}$ in problem (\ref{bii}a), $i$-th objective can be expressed as $\mathcal{L}_i(\bm{x},y,{\bm{\theta}_k}(\bm{\lambda})), \forall i \in \{1, 2\}$ over dataset $\mathcal{D}_{k}$. First, we analyze the existence of the Pareto optimal model. 
\begin{lemma}[Existence]\label{exist}
    Given a preference vector $\bm{\lambda} \in \mathbb{R}^{2}_{\geq 0}$, problem (\ref{bii}) has at least one Pareto optimal solution.
\end{lemma}

Lemma \ref{exist} shows that PraFFL has Pareto optimal solutions for any client's preference $\bm{\lambda}$. After we determine the existence of Pareto optimality of the model, we then analyze the condition when a model is Pareto optimal.

\begin{lemma}[Pareto Optimality]\label{optima} Given a preference vector $\bm{\lambda} \in \mathbb{R}_{\geq 0}^{2}$, a model is weakly Pareto optimal if it minimizes problem (\ref{bii}). The model is Pareto optimal if and only if the optimal solution of problem (\ref{bii}) is unique.
\end{lemma}

This Lemma shows that PraFFL can achieve the Pareto optimal solution for the preference vector $\bm{\lambda}$ if the conditions in Theorem \ref{optima} are satisfied. Finally, we analyze the ability of PraFFL to find the Pareto front, and we have the following theorem:

\begin{theorem}[Pareto Front]\label{pf}
Let $\bm{\theta}_{k}^{'} \in \Theta$ be an optimal model for a given client's preference. Then, there exists a preference vector $\bm{\lambda} \in \mathbb{R}^{2}_{\geq 0}$ such that we can obtain ${\bm{\theta}_{k}^*} (\bm{\lambda})=\bm{\theta}_{k}'$ from problem (\ref{bii}).
\end{theorem}

Theorem \ref{pf} shows that the models can obtain the entire Pareto front. This is because, for any Pareto optimal model $\bm{\theta}^{'}_{k}$, there is a preference vector $\bm{\lambda}$ that the model ${\bm{\theta}^{*}_{k}}(\bm{\lambda})$ is Pareto optimal.

\subsection{Convergence Analysis}
In this subsection, we discuss the convergence properties of PraFFL. The optimization between the communicated model and the personalized model are correlated. We first prove that the communicated model converges linearly under the personalized FL framework. We then prove that the convergence rate of the personalized model is the same as that of the communicated model. Finally, we conclude that PraFFL also converges linearly. 

The convergence analysis of problem (\ref{bii}) is difficult because the weighted Tchebycheff function is non-differentiable. To overcome this issue, we introduce a \textit{smooth Tchebycheff function} as follows:
\begin{equation}\label{smooth}
    \min_{\bm{\theta}_k}\widetilde{g}(\bm{x},y,\bm{\theta}_k(\bm{\lambda})|\bm{\lambda})=\min_{\bm{\theta}_k}\frac{1}{\gamma} \log \sum_{i\in[2]}e^{\frac{\gamma \mathcal{L}_{i}(\bm{x},y, \bm{\theta}_{k}(\bm{\lambda}))}{\lambda_{i}}},
\end{equation}
where $\gamma > 0$ represents the smooth factor. 

Before presenting convergence results, we list the following two commonly used assumptions.
\begin{assumption}\label{bound_g}
    The expectation of stochastic gradients is always bounded, i.e., for all $\bm{\phi}_{k}^{t} \in \Phi$ and $\bm{\psi}^{t} \in \Psi$, we have $\mathbb{E}[\|\nabla \widetilde{g} \ \|^2] \leq G^2$.
\end{assumption}
\begin{assumption}\label{ass2}
$\mathcal{L}_{i}(\bm{x},y,\bm{\theta}_k(\bm{\lambda}))$ is $\mu_{i}$-strongly convex and $L_{i}$-smooth for all $k \in [K]$, $\bm{\lambda} \in \Lambda$, and $\bm{\theta}_k(\bm{\lambda}) \in \Theta$.
\end{assumption}
 Based on the above assumptions, we can deduce the strong convex coefficient and smoothness coefficient of the smooth Tchebycheff function.

\begin{lemma}[Convexity of Smooth Tchebycheff Function]\label{lemmaconvex}
If Assumption \ref{ass2} is satisfied for $i \in [1, 2]$, then $\widetilde{g}(\bm{x},y, \bm{\theta}_k(\bm{\lambda}))$ is $2log(\sum_{i\in[2]}e^{\frac{\mu_{i}}{2\lambda_{i}}})$-strongly convex function, for $k \in [K]$ and $\bm{\lambda} \in \Lambda$. 
\end{lemma}
\begin{lemma}[Smoothness of Smooth Tchebycheff Function]\label{lemmasmooth}
If Assumption \ref{ass2} is satisfied for $i \in [1, 2]$, then $\widetilde{g}(\bm{x},y,\bm{\theta}_k(\bm{\lambda}))$ is $\sum_{i\in[2]}\frac{L_{i}}{\lambda_{i}}$-smooth, for $k \in [K]$. 
\end{lemma}

\begin{lemma}[Convergence of the Communicated Model \cite{collins2021exploiting}]\label{commconv}
    If the communicated model $\bm{\psi}$ is optimized by Fedavg \cite{mcmahan2017communication} and given a constant $\zeta > 0$, then $\bm{\psi}$ converges to optimal communicated model $\bm{\psi}^*$ at a linear rate
    \begin{equation}
        z(t)=\mathbb{E}[\|\bm{\psi}^t-\bm{\psi}^*\|^2] \leq (1-\eta \zeta)^{t/2} \  \mathbb{E}[\|\bm{\psi}^1-\bm{\psi}^* \|],
    \end{equation}
\textit{with a probability no smaller than $1-te^{-100\min(|\mathcal{X}|^2\log(|\mathcal{S}_t|), d)}$}.
\end{lemma}

Lemma \ref{commconv} implies that in order to achieve a bound of $(1-\eta \zeta)^{t/2} \  \mathbb{E}[\|\bm{\psi}^1-\bm{\psi}^* \|] \leq \varepsilon$, we must run $O(\log(\frac1 \varepsilon))$ iterations of gradient descent. This convergence rate is often called \textit{linear convergence}.

 For simplicity, let us denote $\check{\mu}=2log(\sum_{i\in[2]}e^{\frac{\mu_{i}}{2\lambda_{i}}})$ and $\check{L}=\sum_{i\in[2]}\frac{L_{i}}{\lambda_{i}}$. The convergence speed of the personalized model is given by the following proposition.
\begin{proposition}[Progress of one step]\label{relation}
    Under Assumptions \ref{bound_g} and \ref{ass2}, let client $k$ get selected with probability $p_k$ at each communication round $t$ with decaying $\eta_t=\frac{2\check{\mu}}{(t+1)p_k}$, we have
\begin{align}
    \begin{aligned}
        \mathbb{E}&\left[\|\bm{\phi}_k^{t+1}-\bm{\phi}_k^*\|^2\right] \\ \leq  &(1\!-\!\frac{2}{t+1})\mathbb{E}\left[\|\bm{\phi}_k^t\!-\bm{\phi}_k^*\|^2\right]\! +\! \frac{8\check{\mu}^2G\check{L}}{(t+1)^2p_k^2}  \sqrt{\mathbb{E}\left[\|\bm{\psi}^{t}\!-\bm{\psi}^{*}\|^2\right]} \\ 
        & + \!\frac{4\check{\mu}^2G^{2}}{(t+1)^2p_k^2}\! +\!(\frac{4\check{\mu}^2\check{L}^2}{(t+1)^2p_k^2}\! + \!\frac{2\check{\mu}}{t+1}(\check{L}-\check{\mu}))\mathbb{E}\left[\|\bm{\psi}^{t}-\bm{\psi}^{*}\|^2\right] \\
        & +\frac{4\check{\mu}\check{L}}{t+1} \sqrt{\mathbb{E}\left[\|\bm{\phi}_k^t-\bm{\phi}_k^*\|^2\right] \mathbb{E} \left[\|\bm{\psi}^t-\bm{\psi}^*\|^2\right]}. 
    \end{aligned}
\end{align}
\end{proposition}

Proposition \ref{relation} shows the convergence speed of the personalized model is affected by the convergence speed of the communicated model, which is intuitively reasonable because these two types of models are optimized sequentially in our system. We further prove that each client's model $\bm{\theta}_k(\bm{\lambda})$ also is linear convergence.

\begin{theorem}[Convergence of PraFFL]\label{rela} 
If there exists a constant $A$ such that $\frac{z(t+1)}{z(t)}\geq1-\frac{z(t)}{A}$, for $\bm{\lambda} \in \Lambda$, then we have
\begin{equation}
\mathbb{E}[\|\bm{\theta}_k^t(\bm{\lambda})-\bm{\theta}^*_k(\bm{\lambda})\|^2]=O(\log(\frac{1}{t})).
\end{equation}
\textit{with a probability no smaller than $1-te^{-100\min(|\mathcal{X}|^2log(|\mathcal{S}_t|), d)}$}.
\end{theorem}
Theorem \ref{rela} shows that PraFFL linearly converges to Pareto optimal solution $\bm{\theta}^*_k(\bm{\lambda})$.  
\section{NUMERICAL Experiments}\label{exp}
In this section, we emprically compare our proposed PraFFL algorithm with six advanced baselines (LFT+Fedavg \cite{konecny2016federated}, LFT+Ensemble \cite{zeng2021improving}, Agnosticfair \cite{du2021fairness}, FairFed \cite{ezzeldin2023fairfed}, FedFB \cite{zeng2021improving}, and EquiFL \cite{makhija2024achieving}).
\begin{table}[t]
\renewcommand\arraystretch{1}
\caption{Local/Global HV results of seven different algorithms on four datasets. Local HV is evaluated on the local client datasets, and Global HV is evaluated on the global dataset. The best results are highlighted in bold, while the second-best results are underlined.} \label{tab}
\centering
\setlength{\tabcolsep}{1.8mm}{
\begin{tabular}{cccccc}
\hline
\multirow{2}{*}{{\raisebox{-0.15cm}{Method}}} & \multicolumn{4}{c}{\raisebox{-0.05cm}{Global HV $\uparrow$}}    \\ \cmidrule(r){2-5} & {\raisebox{0cm}{SYNTHETIC}}      & {\raisebox{0cm}{COMPAS}}         & {\raisebox{0cm}{BANK}}             & {\raisebox{0cm}{ADULT}}   \\ \cmidrule(r){1-5}
{LFT+Ensemble}            & 0.479          & 0.555          & 0.881          & 0.764      \\
{LFT+Fedavg}              & \underline{0.698}          & 0.539          & \underline{0.883}          & 0.768       \\
{Agnosticfair}            & 0.530          & 0.546          & 0.882          & \underline{0.783}      \\
{FairFed}                 & 0.342          & 0.507          & 0.878          & 0.269       \\
{FedFB}                   & 0.599          & 0.517          & \underline{0.883}          & 0.764      \\ 
{EquiFL}                   & 0.642          & \underline{0.564}          & 0.877          & 0.761      \\ 
\midrule
{PraFFL}              & {\textbf{0.763}} & {\textbf{0.629}} & {\textbf{0.889}} & {\textbf{0.828}}    \\ \midrule 
\multirow{2}{*}{{\raisebox{-0.15cm}{Method}}}
  & \multicolumn{4}{c}{\raisebox{-0.05cm}{Local HV $\uparrow$}}                       \\ \cmidrule(r){2-5}
                        & {\raisebox{0cm}{SYNTHETIC}}      & {\raisebox{0cm}{COMPAS}}         & {\raisebox{0cm}{BANK}}             & {\raisebox{0cm}{ADULT}}        \\ \midrule 
                        LFT+Ensemble &0.468 & 0.514 & 0.138 & 0.501 \\
                        LFT+Fedavg&0.479 & \underline{0.555} & 0.881 & 0.764 \\
                        Agnosticfair& 0.530 & 0.551 & 0.880 & \underline{0.784} \\
                        FairFed& 0.342 & 0.507 & 0.878 & 0.269 \\
                        FedFB& 0.599 & 0.517 & \underline{0.883} &  0.763 \\
                         EquiFL & \underline{0.604} & 0.526 & 0.882 & 0.764 \\
                        \midrule
                        PraFFL&\textbf{0.778} & \textbf{0.631} & \textbf{0.895} & \textbf{0.834} \\ \midrule
\end{tabular}}
\end{table}
\begin{figure}[t]
    \centering
    \includegraphics[width=\linewidth]{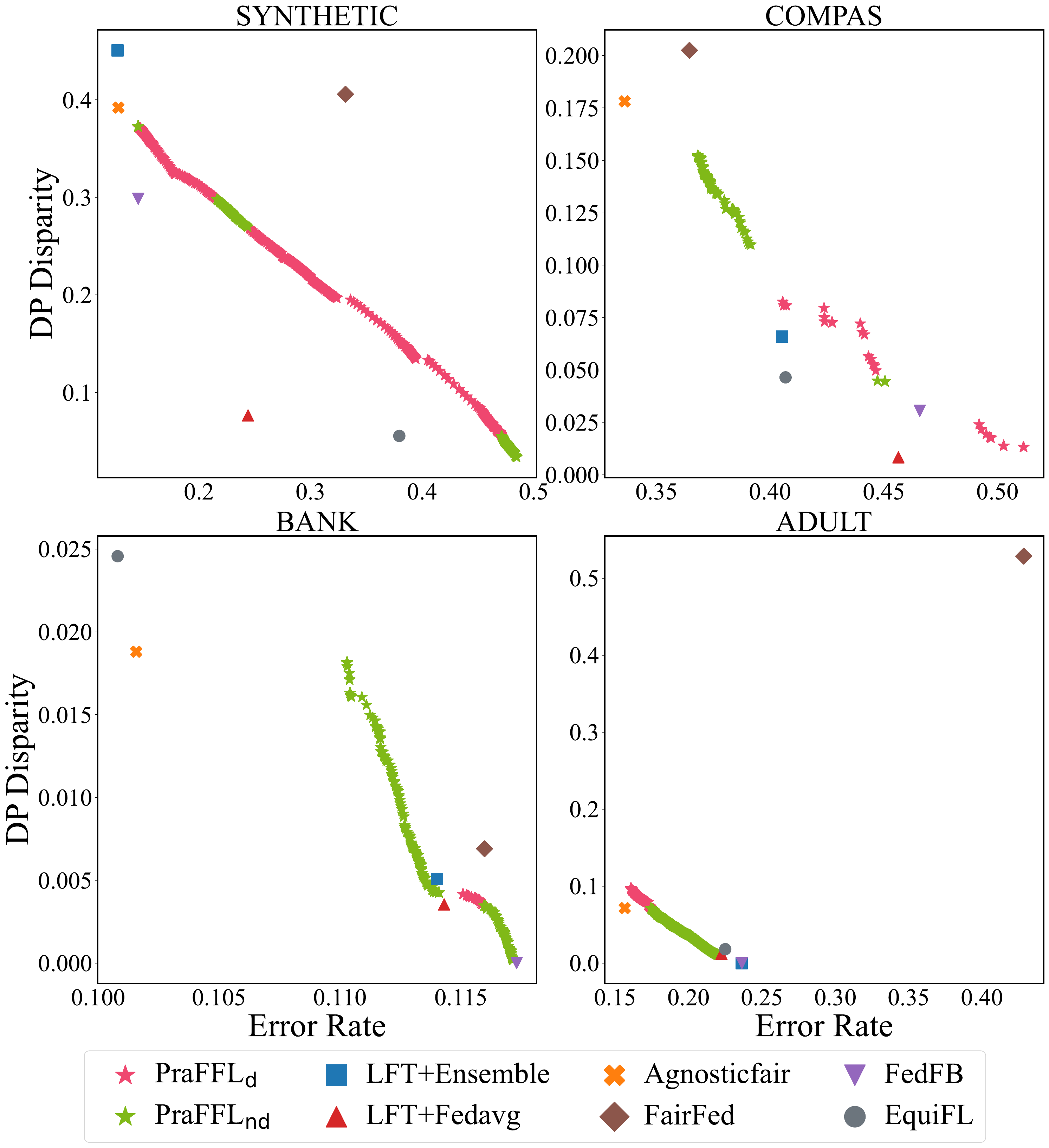}
  \caption{Comparison of solutions obtained by six algorithms on four datasets. We divide the solution set obtained by the PraFFL method into two categories: {\color{d}$\text{PraFFL}_{d}$} and {\color{nd}$\text{PraFFL}_{nd}$}. {\color{d}$\text{PraFFL}_{d}$} means that the solution is dominated (i.e., there is another algorithm that is better than the solutions in {\color{d}$\text{PraFFL}_{d}$}), and {\color{nd}$\text{PraFFL}_{nd}$} means that the solution is non-dominated (i.e., there is no other solution that is better than the solutions in {\color{nd}$\text{PraFFL}_{nd}$}).}
  \label{solutionset}
\end{figure}

\subsection{Experimental Settings}
We report results from three runs on each dataset. PraFFL shows the solutions obtained by 1000 preference vectors in the experiment (in fact, any number of solutions can be obtained through PraFFL inference if a corresponding number of preference vectors are given). The total local epochs are set to 30 (i.e., $\tau_{p}+{\tau_{c}=30}$). We use hypervolume (HV) \cite{fonseca2006improved} to measure the quality of the solution set obtained by each algorithm. We mainly report local HV on local datasets and global HV on global datasets in our experiments. The higher the quality of the solutions in satisfying the client's preferences, the larger the HV value. Four widely used datasets are employed to conduct experiments, including SYNTHETIC \cite{zeng2021improving}, COMPAS \cite{barenstein2019propublica}, BANK \cite{moro2014data} and ADULT \cite{dua2017uci}.
\subsection{Main Results}
Inspired by \cite{hamman2024demystifying,cui2021addressing}, we evaluate all algorithms on local and global datasets. The results presented in Table \ref{tab} demonstrate that PraFFL is consistently superior to the performance of five baselines regarding HV. Our proposed PraFFL has the best local HV value on the four datasets, which are 0.778, 0.631, 0.895, and 0.834, respectively. This is because our proposed method can infer any number of personalized models for clients, thereby obtaining more solutions with different trade-offs between model performance and fairness. Additionally, Table \ref{tab} also shows the global HV of each algorithm. \rev{Although PraFFL is designed for local Pareto front, it can achieve the best global HV compared with six other baselines} Recall that other baselines require retraining the model if each client has more than one preference, and the training time increases linearly with the number of preferences (see Table \ref{mm}). 

\begin{figure*}[tp]
    \centering
    \includegraphics[width=0.95\textwidth]{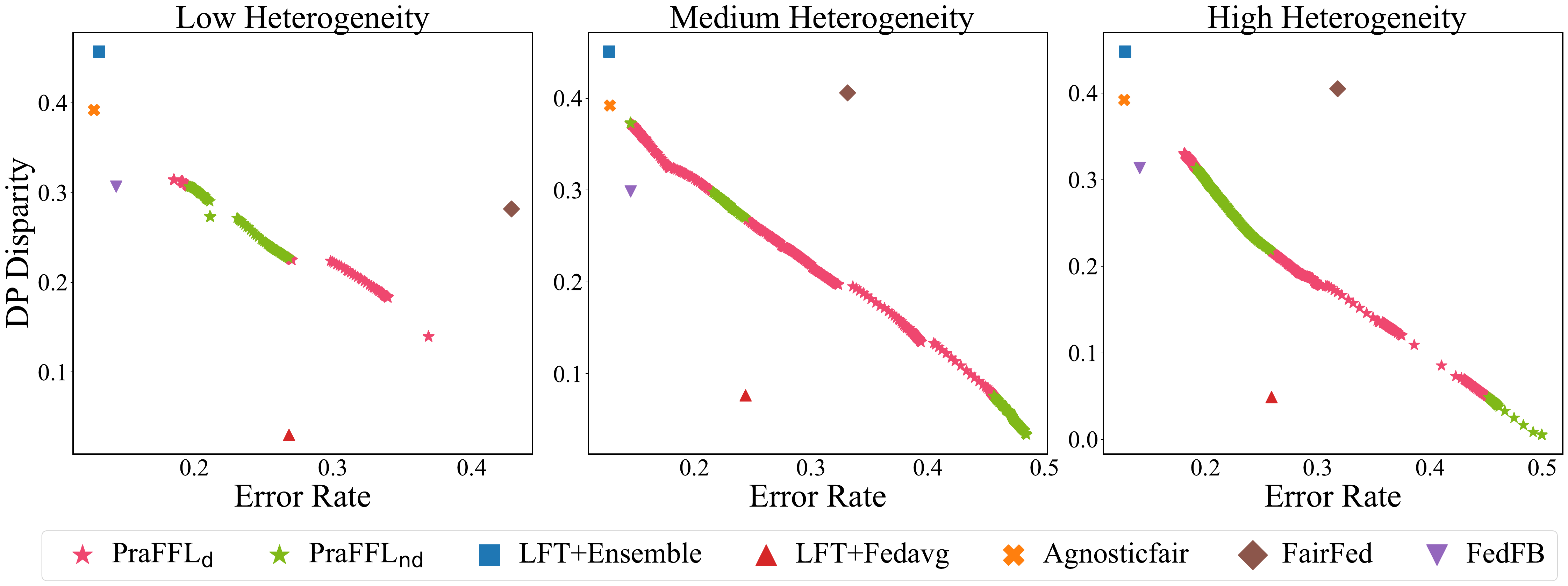}
  \caption{Accuracy-fairness solutions on the SYNTHETIC dataset with three clients having different heterogeneity distributions of the sensitive feature. Left: low heterogeneity with Dirichlet distribution $\text{Diri}(100)$. Middle: medium heterogeneity with with Dirichlet distribution $\text{Diri}(5)$. Right: high heterogeneity with with Dirichlet distribution $\text{Diri}(0.1)$.} \label{hetero}
\end{figure*}

Moreover, Fig. \ref{solutionset} shows the final solutions obtained by the baselines and the non-dominated solutions obtained by PraFFL among 1000 solutions. According to the dominance relationship in Definition \ref{dominated}, we divide the solutions obtained by PraFFL into two categories ($\text{PraFFL}_d$ and $\text{PraFFL}_{nd}$). $\text{PraFFL}_d$ (red stars) are inferior to some baselines in some preferences. The solution set in $\text{PraFFL}_{nd}$ (green stars) is superior to other baselines (it is the best trade-off between fairness and error rate for some specific preferences). Since the solution set of PraFFL covers the optimal trade-offs for most of the client's preferences, the solution set obtained by PraFFL is better than the solution obtained by other algorithms. The advantage of PraFFL is that it can provide arbitrary solutions for clients with only one training. \rev{In addition, the Pareto front obtained by PraFFL on the COMPAS and BANK datasets can be further improved. There are two aspects that can further improve the learned Pareto front: (I) The data distribution of the training and testing set are not exactly the same. Obtaining a superior solution set on the training set may not necessarily lead to a superior solution set on the testing set; (II) Our optimization goals of model performance and fairness are not directly equivalent to optimizing the error rate and DP disparity.} Therefore, these two aspects lead to a gap between the solution set we obtained and the optimal solution set, which motivates future research directions. 

\subsection{Ablation Studies}
\rev{This subsection mainly investigated the effects of data heterogeneity, number of clients, and hyperparameters in PraFFL.}

\subsubsection{The effect of the number of clients\label{moreclient}}
Table \ref{number} shows the results under different numbers of clients on the SYNTHETIC dataset. Two observations can be made: (I) Regardless of the large-scale number of clients, our proposed PraFFL achieves the best local HV and global HV compared to the five baseline algorithms, with PraFFL leading the best algorithm by 2\% to 18\%. This demonstrates that our PraFFL is highly effective in providing preference-specific models for clients in large-scale scenarios. (II) As the number of clients increases, there is a slight decrease in PraFFL's local HV and global HV. Nevertheless, PraFFL still achieves state-of-the-art performance compared with five baselines. This phenomenon could be attributed to the increasing divergence in local Pareto fronts with the growing number of clients, due to the differences in local datasets across clients. As a result, this leads to the emergence of heterogeneous local Pareto fronts.
\subsubsection{The effect of data heterogeneity}\label{heter}
We examine the error rate and DP Disparity obtained by each algorithm under different heterogeneous levels on the sensitive features. As in \cite{zeng2021improving}, we study three heterogeneous scenarios. Fig. \ref{hetero} shows the performance of the algorithms under three different heterogeneity levels. We can observe that our proposed PraFFL are robust than the other five baselines under different heterogeneous levels. Compared with the other five algorithms, the number of solutions of $\text{PraFFL}_{nd}$ is larger, and such solutions are more widely distributed in different heterogeneities. For four baselines, the changes in the values of their two indicators under the three levels of heterogeneity are relatively small. It is worth noting that PraFFL is robust to large-scale user preferences. In other words, PraFFL not only have considerable performance in meeting users' dynamic preferences, but also maintain high stability under different degrees of data heterogeneity. 
\begin{figure}[t]
    \centering
    \includegraphics[width=\linewidth]{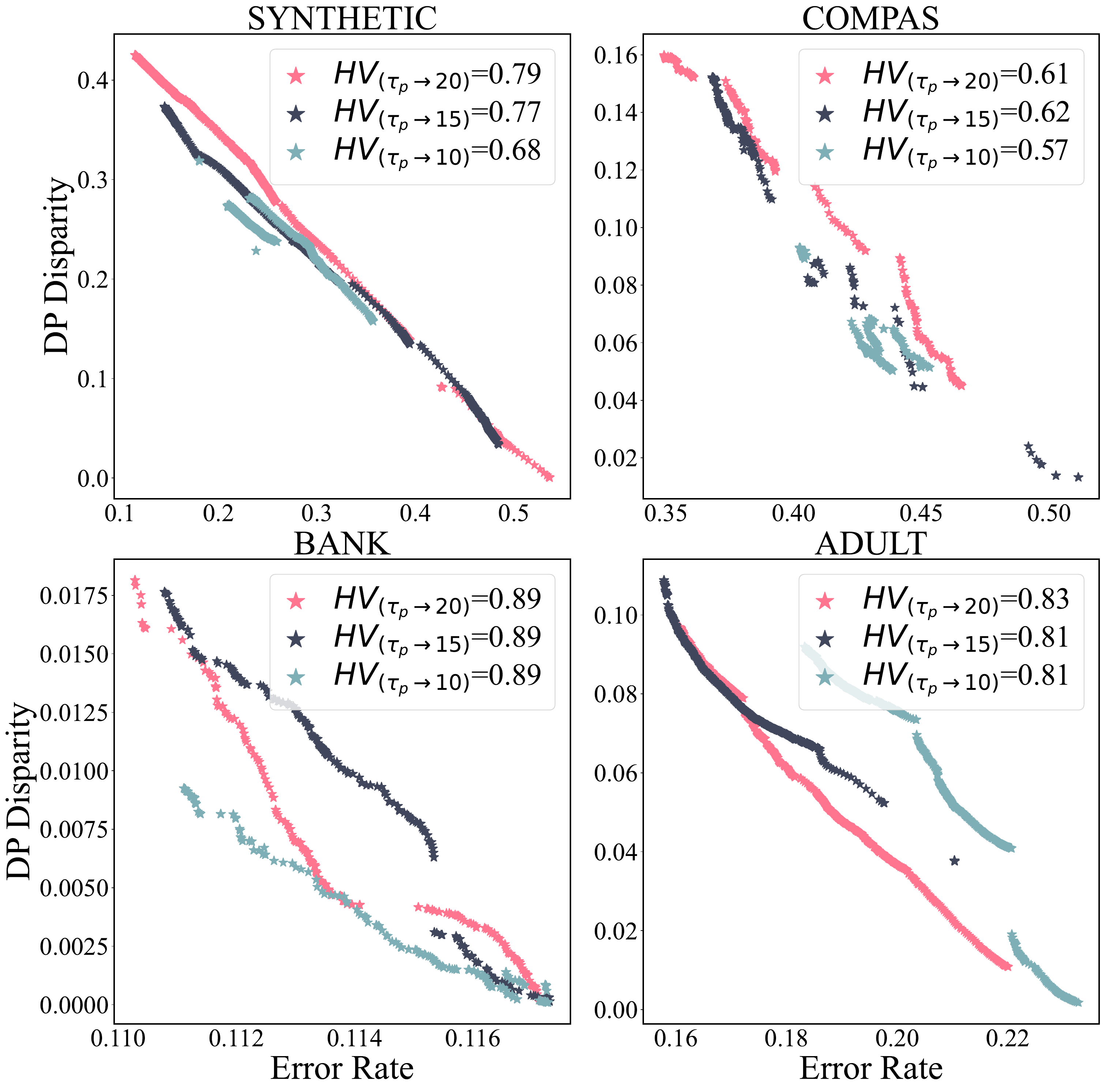}
    \caption{Solutions obtained by PraFFL under different $\tau_{p}$.} 
    \label{localtrain}
\end{figure}
\begin{table}[t]
    \centering
    \setlength{\abovecaptionskip}{0.12cm}
    \setlength{\tabcolsep}{4pt}
    \caption{Performance comparison of different methods across the number of clients on SYNTHETIC dataset.}\label{number}
    \begin{tabular}{lccccccc}
        \toprule
        \multirow{3}{*}{Method} & \multicolumn{3}{c}{Local HV $\uparrow$} & \multicolumn{3}{c}{Global HV $\uparrow$} \\
        \cmidrule(lr){2-4} \cmidrule(lr){5-7}
                & \multicolumn{3}{c}{Number of Clients} & \multicolumn{3}{c}{Number of Clients}\\
        & \multicolumn{1}{c}{10} & \multicolumn{1}{c}{50} & \multicolumn{1}{c}{100} & \multicolumn{1}{c}{10} & \multicolumn{1}{c}{50} & \multicolumn{1}{c}{100} \\
        \midrule
        LFT+Ensemble & 0.475  & 0.599 & 0.589 & 0.459 & 0.475  & 0.464 \\
        LFT+Fedavg & \underline{0.663}  & \underline{0.614} & 0.536 & 0.463 & 0.589  & 0.472 \\
        Agnosticfair & 0.548  & 0.597 & 0.623 & 0.551 & 0.559  & 0.556 \\
        FairFed & 0.415  & 0.394 & 0.388 & 0.372 & 0.309  & 0.329 \\
        FedFB & 0.597  & \underline{0.614} & \underline{0.649} & \underline{0.613} & \underline{0.591}  & \underline{0.589} \\
        PraFFL & \textbf{0.786}  & \textbf{0.695} & \textbf{0.676} & \textbf{0.776} & \textbf{0.661}  & \textbf{0.654} \\
        \bottomrule
    \end{tabular}
\end{table}
\subsubsection{The effect of $\tau_{p}$}
HV is calculated by the solution set and a reference point (Fig. \ref{fig:hv}). HV is affected by two factors: (1) quality of the solution set (i.e., the closer to the origin, the higher HV); and (2) diversity of the solution set (i.e., the broader the solutions distributed, the higher HV). Recall that $\tau_p + \tau_c$ equals the total local training epochs. As shown in Fig. \ref{localtrain}, adjusting the value of $\tau_p$ is a trade-off problem. We can find that
larger $\tau_{p}$ increases the diversity of the solution set, while it leads to a smaller 
$\tau_c$, making the quality of the solution set worse. This is because the training of the communicated model mainly focuses on improving the solution set (Eq. (\ref{communicated})),
and a larger $\tau_{p}$ will allow the personalized model to fully learn the user's preferences. 
Conversely, a smaller $\tau_p$ decreases the diversity but improves the quality of the solution set. 
For instance, when $\tau_{p}=10$ in COMPAS, the convergence results of the solution set are superior, but the diversity of the solution set is inferior.
\begin{figure}[t]
    \centering
    \includegraphics[width=\linewidth]{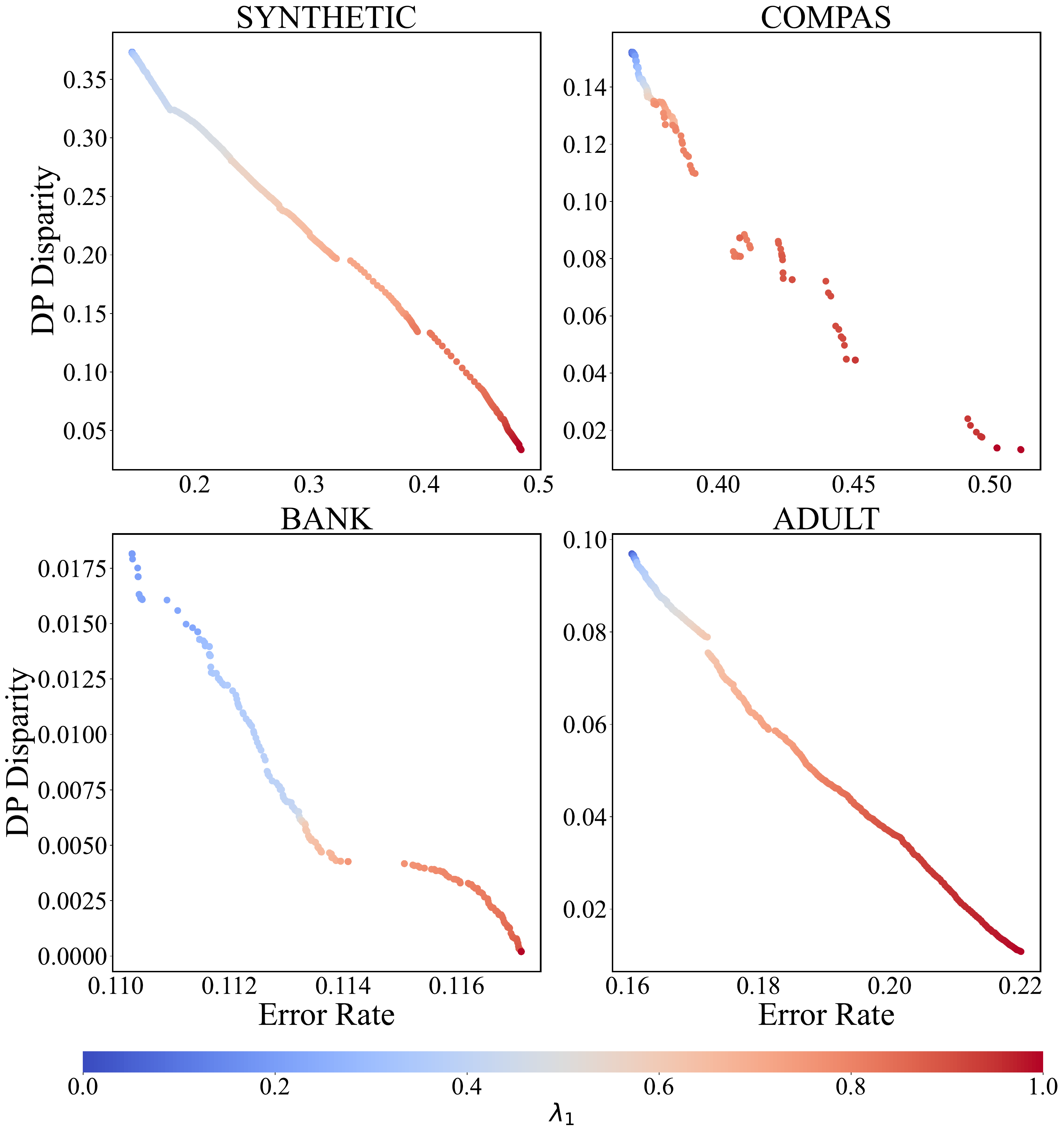}
    \caption{Mapping relationship from different preferences to model performance and fairness.}
    \label{pref}
\end{figure}
\subsubsection{Different preferences of clients}
We analyze the mapping from different preference vectors to their corresponding solution set on four datasets. For the simplicity of verifying the effectiveness of PraFFL, we assume that the preferences of all users are unchanged during inference. In Fig. \ref{pref}, we can find that the solutions corresponding to different preferences are clearly distinguishable. This shows that PraFFL can assign a unique solution to each user's different preferences. \rev{In addition, we observe that the Pareto front obtained by PraFFL on SYNTHETIC, BANK, and ADULT have a wide range, which indicates that PraFFL provide models for various preferences of clients.} The learned Pareto front under COMPAS is slightly worse than that under the other three datasets, which may be due to data quality. In general, PraFFL can meet most of the preference requirements of clients based on the preference vectors they provide.

\section{Conclusion and Future Works}\label{conclusion}
 In this work, we proposed a preference-aware scheme in fair federated learning (PraFFL), which is capable of providing models with different trade-offs to every client in real time. In the training phase, our proposed PraFFL does not require the assumption that the Pareto front is convex and preserves the client's preference privacy. In the inference phase, the model obtained by PraFFL can be adaptively adjusted according to the client's preferences. We theoretically prove that given a client's preference, PraFFL can linearly converge to the Pareto optimal model and learn the entire Pareto front. Experimental results show that the solution set obtained by PraFFL on four datasets is well-performing, and the solutions obtained by PraFFL are the optimal choices for most of the client's preferences compared to five advanced algorithms.

Although our proposed PraFFL can achieve promising results on the four datasets, it still has room for improvement. The first direction is to improve the generalization ability of the learned Pareto front. The distribution of the training and testing dataset is inconsistent. A good solution set on the training dataset may not lead to a superior solution set on the testing dataset. Therefore, the generalization of the learned Pareto front can be improved by regularization or using advanced data sampling strategies. In the second direction, the cross-entropy and fairness loss functions used in our work are not equivalent to the error rate and DP disparity formulations, respectively. To further improve the performance of PraFFL, a possible approach is to approximate the error rate and DP disparity formulations.

\section*{Acknowledgments}
This work was supported by the National Natural Science Foundation of China under Grant 62202214 and Guangdong Basic and Applied Basic Research Foundation under Grant 2023A1515012819. 

\clearpage
\bibliographystyle{ACM-Reference-Format}
\bibliography{sample-base}


\begin{thebibliography}{49}


\ifx \showCODEN    \undefined \def \showCODEN     #1{\unskip}     \fi
\ifx \showDOI      \undefined \def \showDOI       #1{#1}\fi
\ifx \showISBNx    \undefined \def \showISBNx     #1{\unskip}     \fi
\ifx \showISBNxiii \undefined \def \showISBNxiii  #1{\unskip}     \fi
\ifx \showISSN     \undefined \def \showISSN      #1{\unskip}     \fi
\ifx \showLCCN     \undefined \def \showLCCN      #1{\unskip}     \fi
\ifx \shownote     \undefined \def \shownote      #1{#1}          \fi
\ifx \showarticletitle \undefined \def \showarticletitle #1{#1}   \fi
\ifx \showURL      \undefined \def \showURL       {\relax}        \fi
\providecommand\bibfield[2]{#2}
\providecommand\bibinfo[2]{#2}
\providecommand\natexlab[1]{#1}
\providecommand\showeprint[2][]{arXiv:#2}

\bibitem[Abay et~al\mbox{.}(2020)]%
        {abay2020mitigating}
\bibfield{author}{\bibinfo{person}{Annie Abay}, \bibinfo{person}{Yi Zhou}, \bibinfo{person}{Nathalie Baracaldo}, \bibinfo{person}{Shashank Rajamoni}, \bibinfo{person}{Ebube Chuba}, {and} \bibinfo{person}{Heiko Ludwig}.} \bibinfo{year}{2020}\natexlab{}.
\newblock \showarticletitle{Mitigating bias in federated learning}.
\newblock \bibinfo{journal}{\emph{arXiv preprint arXiv:2012.02447}} (\bibinfo{year}{2020}).
\newblock


\bibitem[Barenstein(2019)]%
        {barenstein2019propublica}
\bibfield{author}{\bibinfo{person}{Matias Barenstein}.} \bibinfo{year}{2019}\natexlab{}.
\newblock \showarticletitle{Propublica's compas data revisited}.
\newblock \bibinfo{journal}{\emph{arXiv preprint arXiv:1906.04711}} (\bibinfo{year}{2019}).
\newblock


\bibitem[Boyd and Vandenberghe(2004)]%
        {boyd2004convex}
\bibfield{author}{\bibinfo{person}{Stephen~P Boyd} {and} \bibinfo{person}{Lieven Vandenberghe}.} \bibinfo{year}{2004}\natexlab{}.
\newblock \bibinfo{booktitle}{\emph{Convex optimization}}.
\newblock \bibinfo{publisher}{Cambridge University Press}.
\newblock


\bibitem[Chu et~al\mbox{.}(2021)]%
        {chu2021fedfair}
\bibfield{author}{\bibinfo{person}{Lingyang Chu}, \bibinfo{person}{Lanjun Wang}, \bibinfo{person}{Yanjie Dong}, \bibinfo{person}{Jian Pei}, \bibinfo{person}{Zirui Zhou}, {and} \bibinfo{person}{Yong Zhang}.} \bibinfo{year}{2021}\natexlab{}.
\newblock \showarticletitle{Fedfair: Training fair models in cross-silo federated learning}.
\newblock \bibinfo{journal}{\emph{arXiv preprint arXiv:2109.05662}} (\bibinfo{year}{2021}).
\newblock


\bibitem[Coello(2007)]%
        {coello2007evolutionary}
\bibfield{author}{\bibinfo{person}{Carlos A~Coello Coello}.} \bibinfo{year}{2007}\natexlab{}.
\newblock \bibinfo{booktitle}{\emph{Evolutionary algorithms for solving multi-objective problems}}.
\newblock \bibinfo{publisher}{Springer}.
\newblock


\bibitem[Collins et~al\mbox{.}(2021)]%
        {collins2021exploiting}
\bibfield{author}{\bibinfo{person}{Liam Collins}, \bibinfo{person}{Hamed Hassani}, \bibinfo{person}{Aryan Mokhtari}, {and} \bibinfo{person}{Sanjay Shakkottai}.} \bibinfo{year}{2021}\natexlab{}.
\newblock \showarticletitle{Exploiting shared representations for personalized federated learning}. In \bibinfo{booktitle}{\emph{Proceedings of the International Conference on Machine Learning}}. PMLR, \bibinfo{pages}{2089--2099}.
\newblock


\bibitem[Cui et~al\mbox{.}(2021)]%
        {cui2021addressing}
\bibfield{author}{\bibinfo{person}{Sen Cui}, \bibinfo{person}{Weishen Pan}, \bibinfo{person}{Jian Liang}, \bibinfo{person}{Changshui Zhang}, {and} \bibinfo{person}{Fei Wang}.} \bibinfo{year}{2021}\natexlab{}.
\newblock \showarticletitle{Addressing algorithmic disparity and performance inconsistency in federated learning}. In \bibinfo{booktitle}{\emph{Proceedings of Advances in Neural Information Processing Systems}}, Vol.~\bibinfo{volume}{34}. \bibinfo{pages}{26091--26102}.
\newblock


\bibitem[Deng et~al\mbox{.}(2020)]%
        {deng2020distributionally}
\bibfield{author}{\bibinfo{person}{Yuyang Deng}, \bibinfo{person}{Mohammad~Mahdi Kamani}, {and} \bibinfo{person}{Mehrdad Mahdavi}.} \bibinfo{year}{2020}\natexlab{}.
\newblock \showarticletitle{Distributionally robust federated averaging}. In \bibinfo{booktitle}{\emph{Proceedings of the Advances in Neural Information Processing Systems}}, Vol.~\bibinfo{volume}{33}. \bibinfo{pages}{15111--15122}.
\newblock


\bibitem[Du et~al\mbox{.}(2021)]%
        {du2021fairness}
\bibfield{author}{\bibinfo{person}{Wei Du}, \bibinfo{person}{Depeng Xu}, \bibinfo{person}{Xintao Wu}, {and} \bibinfo{person}{Hanghang Tong}.} \bibinfo{year}{2021}\natexlab{}.
\newblock \showarticletitle{Fairness-aware agnostic federated learning}. In \bibinfo{booktitle}{\emph{Proceedings of the SIAM International Conference on Data Mining (SDM)}}. SIAM, \bibinfo{pages}{181--189}.
\newblock


\bibitem[Dua et~al\mbox{.}(2017)]%
        {dua2017uci}
\bibfield{author}{\bibinfo{person}{Dheeru Dua}, \bibinfo{person}{Casey Graff}, {et~al\mbox{.}}} \bibinfo{year}{2017}\natexlab{}.
\newblock \showarticletitle{UCI machine learning repository, 2017}.
\newblock \bibinfo{journal}{\emph{URL http://archive. ics. uci. edu/ml}} \bibinfo{volume}{7}, \bibinfo{number}{1} (\bibinfo{year}{2017}), \bibinfo{pages}{62}.
\newblock


\bibitem[Ezzeldin et~al\mbox{.}(2023)]%
        {ezzeldin2023fairfed}
\bibfield{author}{\bibinfo{person}{Yahya~H Ezzeldin}, \bibinfo{person}{Shen Yan}, \bibinfo{person}{Chaoyang He}, \bibinfo{person}{Emilio Ferrara}, {and} \bibinfo{person}{A~Salman Avestimehr}.} \bibinfo{year}{2023}\natexlab{}.
\newblock \showarticletitle{Fairfed: Enabling group fairness in federated learning}. In \bibinfo{booktitle}{\emph{Proceedings of the AAAI Conference on Artificial Intelligence}}, Vol.~\bibinfo{volume}{37}. \bibinfo{pages}{7494--7502}.
\newblock


\bibitem[Feldman et~al\mbox{.}(2015)]%
        {feldman2015certifying}
\bibfield{author}{\bibinfo{person}{Michael Feldman}, \bibinfo{person}{Sorelle~A Friedler}, \bibinfo{person}{John Moeller}, \bibinfo{person}{Carlos Scheidegger}, {and} \bibinfo{person}{Suresh Venkatasubramanian}.} \bibinfo{year}{2015}\natexlab{}.
\newblock \showarticletitle{Certifying and removing disparate impact}. In \bibinfo{booktitle}{\emph{Proceedings of the ACM SIGKDD International Conference on Knowledge Discovery and Data Mining}}. \bibinfo{pages}{259--268}.
\newblock


\bibitem[Fonseca et~al\mbox{.}(2006)]%
        {fonseca2006improved}
\bibfield{author}{\bibinfo{person}{Carlos~M Fonseca}, \bibinfo{person}{Lu{\'\i}s Paquete}, {and} \bibinfo{person}{Manuel L{\'o}pez-Ib{\'a}nez}.} \bibinfo{year}{2006}\natexlab{}.
\newblock \showarticletitle{An improved dimension-sweep algorithm for the hypervolume indicator}. In \bibinfo{booktitle}{\emph{Proceedings of the International Conference on Evolutionary Computation}}. IEEE, \bibinfo{pages}{1157--1163}.
\newblock


\bibitem[G{\'a}lvez et~al\mbox{.}(2021)]%
        {galvez2021enforcing}
\bibfield{author}{\bibinfo{person}{Borja~Rodr{\'\i}guez G{\'a}lvez}, \bibinfo{person}{Filip Granqvist}, \bibinfo{person}{Rogier van Dalen}, {and} \bibinfo{person}{Matt Seigel}.} \bibinfo{year}{2021}\natexlab{}.
\newblock \showarticletitle{Enforcing fairness in private federated learning via the modified method of differential multipliers}. In \bibinfo{booktitle}{\emph{Proceedings of NeurIPS Workshop Privacy in Machine Learning}}.
\newblock


\bibitem[Gu et~al\mbox{.}(2022)]%
        {gu2022privacy}
\bibfield{author}{\bibinfo{person}{Xiuting Gu}, \bibinfo{person}{Zhu Tianqing}, \bibinfo{person}{Jie Li}, \bibinfo{person}{Tao Zhang}, \bibinfo{person}{Wei Ren}, {and} \bibinfo{person}{Kim-Kwang~Raymond Choo}.} \bibinfo{year}{2022}\natexlab{}.
\newblock \showarticletitle{Privacy, accuracy, and model fairness trade-offs in federated learning}.
\newblock \bibinfo{journal}{\emph{Computers \& Security}}  \bibinfo{volume}{122} (\bibinfo{year}{2022}), \bibinfo{pages}{102907}.
\newblock


\bibitem[Hamman and Dutta(2024)]%
        {hamman2024demystifying}
\bibfield{author}{\bibinfo{person}{Faisal Hamman} {and} \bibinfo{person}{Sanghamitra Dutta}.} \bibinfo{year}{2024}\natexlab{}.
\newblock \showarticletitle{Demystifying local and global fairness trade-offs in federated learning using partial information decomposition}. In \bibinfo{booktitle}{\emph{Proceedings of the International Conference on Learning Representations}}.
\newblock


\bibitem[Hu et~al\mbox{.}(2022)]%
        {hu2022fair}
\bibfield{author}{\bibinfo{person}{Shengyuan Hu}, \bibinfo{person}{Zhiwei~Steven Wu}, {and} \bibinfo{person}{Virginia Smith}.} \bibinfo{year}{2022}\natexlab{}.
\newblock \showarticletitle{Fair federated learning via bounded group loss}.
\newblock \bibinfo{journal}{\emph{arXiv preprint arXiv:2203.10190}} (\bibinfo{year}{2022}).
\newblock


\bibitem[Jacques et~al\mbox{.}(2023)]%
        {jacques2023stopasianhate}
\bibfield{author}{\bibinfo{person}{Erin~T Jacques}, \bibinfo{person}{Corey~H Basch}, \bibinfo{person}{Joseph Fera}, {and} \bibinfo{person}{Vincent Jones~II}.} \bibinfo{year}{2023}\natexlab{}.
\newblock \showarticletitle{\# StopAsianHate: A content analysis of tiktok videos focused on racial discrimination against asians and asian americans during the COVID-19 pandemic}.
\newblock \bibinfo{journal}{\emph{Dialogues in Health}}  \bibinfo{volume}{2} (\bibinfo{year}{2023}), \bibinfo{pages}{100089}.
\newblock


\bibitem[Jin et~al\mbox{.}(2021)]%
        {jin2021cafe}
\bibfield{author}{\bibinfo{person}{Xiao Jin}, \bibinfo{person}{Pin-Yu Chen}, \bibinfo{person}{Chia-Yi Hsu}, \bibinfo{person}{Chia-Mu Yu}, {and} \bibinfo{person}{Tianyi Chen}.} \bibinfo{year}{2021}\natexlab{}.
\newblock \showarticletitle{Cafe: Catastrophic data leakage in vertical federated learning}. In \bibinfo{booktitle}{\emph{Proceedings of the Advances in Neural Information Processing Systems}}, Vol.~\bibinfo{volume}{34}. \bibinfo{pages}{994--1006}.
\newblock


\bibitem[Kamishima et~al\mbox{.}(2012)]%
        {kamishima2012fairness}
\bibfield{author}{\bibinfo{person}{Toshihiro Kamishima}, \bibinfo{person}{Shotaro Akaho}, \bibinfo{person}{Hideki Asoh}, {and} \bibinfo{person}{Jun Sakuma}.} \bibinfo{year}{2012}\natexlab{}.
\newblock \showarticletitle{Fairness-aware classifier with prejudice remover regularizer}. In \bibinfo{booktitle}{\emph{Proceedings of Machine Learning and Knowledge Discovery in Databases}}. Springer, \bibinfo{pages}{35--50}.
\newblock


\bibitem[Kone{\v{c}}n{\`y} et~al\mbox{.}(2016)]%
        {konevcny2016federated}
\bibfield{author}{\bibinfo{person}{Jakub Kone{\v{c}}n{\`y}}, \bibinfo{person}{H~Brendan McMahan}, \bibinfo{person}{Daniel Ramage}, {and} \bibinfo{person}{Peter Richt{\'a}rik}.} \bibinfo{year}{2016}\natexlab{}.
\newblock \showarticletitle{Federated optimization: Distributed machine learning for on-device intelligence}.
\newblock \bibinfo{journal}{\emph{arXiv preprint arXiv:1610.02527}} (\bibinfo{year}{2016}).
\newblock


\bibitem[Konecn{\`y} et~al\mbox{.}(2016)]%
        {konecny2016federated}
\bibfield{author}{\bibinfo{person}{Jakub Konecn{\`y}}, \bibinfo{person}{H~Brendan McMahan}, \bibinfo{person}{Felix~X Yu}, \bibinfo{person}{Peter Richt{\'a}rik}, \bibinfo{person}{Ananda~Theertha Suresh}, {and} \bibinfo{person}{Dave Bacon}.} \bibinfo{year}{2016}\natexlab{}.
\newblock \showarticletitle{Federated learning: Strategies for improving communication efficiency}.
\newblock \bibinfo{journal}{\emph{arXiv preprint arXiv:1610.05492}}  \bibinfo{volume}{8} (\bibinfo{year}{2016}).
\newblock


\bibitem[Li et~al\mbox{.}(2021)]%
        {li2021ditto}
\bibfield{author}{\bibinfo{person}{Tian Li}, \bibinfo{person}{Shengyuan Hu}, \bibinfo{person}{Ahmad Beirami}, {and} \bibinfo{person}{Virginia Smith}.} \bibinfo{year}{2021}\natexlab{}.
\newblock \showarticletitle{Ditto: Fair and robust federated learning through personalization}. In \bibinfo{booktitle}{\emph{Proceedings of International Conference on Machine Learning}}. PMLR, \bibinfo{pages}{6357--6368}.
\newblock


\bibitem[Li et~al\mbox{.}(2019)]%
        {li2019fair}
\bibfield{author}{\bibinfo{person}{Tian Li}, \bibinfo{person}{Maziar Sanjabi}, \bibinfo{person}{Ahmad Beirami}, {and} \bibinfo{person}{Virginia Smith}.} \bibinfo{year}{2019}\natexlab{}.
\newblock \showarticletitle{Fair resource allocation in federated learning}.
\newblock \bibinfo{journal}{\emph{arXiv preprint arXiv:1905.10497}} (\bibinfo{year}{2019}).
\newblock


\bibitem[Lyu et~al\mbox{.}(2020)]%
        {lyu2020collaborative}
\bibfield{author}{\bibinfo{person}{Lingjuan Lyu}, \bibinfo{person}{Xinyi Xu}, \bibinfo{person}{Qian Wang}, {and} \bibinfo{person}{Han Yu}.} \bibinfo{year}{2020}\natexlab{}.
\newblock \showarticletitle{Collaborative fairness in federated learning}.
\newblock \bibinfo{journal}{\emph{Federated Learning: Privacy and Incentive}} (\bibinfo{year}{2020}), \bibinfo{pages}{189--204}.
\newblock


\bibitem[Makhija et~al\mbox{.}(2024)]%
        {makhija2024achieving}
\bibfield{author}{\bibinfo{person}{Disha Makhija}, \bibinfo{person}{Xing Han}, \bibinfo{person}{Joydeep Ghosh}, {and} \bibinfo{person}{Yejin Kim}.} \bibinfo{year}{2024}\natexlab{}.
\newblock \showarticletitle{Achieving fairness across local and global models in federated learning}.
\newblock \bibinfo{journal}{\emph{arXiv preprint arXiv:2406.17102}} (\bibinfo{year}{2024}).
\newblock


\bibitem[McMahan et~al\mbox{.}(2017)]%
        {mcmahan2017communication}
\bibfield{author}{\bibinfo{person}{Brendan McMahan}, \bibinfo{person}{Eider Moore}, \bibinfo{person}{Daniel Ramage}, \bibinfo{person}{Seth Hampson}, {and} \bibinfo{person}{Blaise~Aguera y Arcas}.} \bibinfo{year}{2017}\natexlab{}.
\newblock \showarticletitle{Communication-efficient learning of deep networks from decentralized data}. In \bibinfo{booktitle}{\emph{Proceedings of the Artificial Intelligence and Statistics}}. PMLR, \bibinfo{pages}{1273--1282}.
\newblock


\bibitem[Mehrabi et~al\mbox{.}(2022)]%
        {mehrabi2022towards}
\bibfield{author}{\bibinfo{person}{Ninareh Mehrabi}, \bibinfo{person}{Cyprien de Lichy}, \bibinfo{person}{John McKay}, \bibinfo{person}{Cynthia He}, {and} \bibinfo{person}{William Campbell}.} \bibinfo{year}{2022}\natexlab{}.
\newblock \showarticletitle{Towards multi-objective statistically fair federated learning}.
\newblock \bibinfo{journal}{\emph{arXiv preprint arXiv:2201.09917}} (\bibinfo{year}{2022}).
\newblock


\bibitem[Miettinen(1999)]%
        {miettinen1999nonlinear}
\bibfield{author}{\bibinfo{person}{Kaisa Miettinen}.} \bibinfo{year}{1999}\natexlab{}.
\newblock \bibinfo{booktitle}{\emph{Nonlinear multiobjective optimization}}. Vol.~\bibinfo{volume}{12}.
\newblock \bibinfo{publisher}{Springer Science \& Business Media}.
\newblock


\bibitem[Mohri et~al\mbox{.}(2019)]%
        {mohri2019agnostic}
\bibfield{author}{\bibinfo{person}{Mehryar Mohri}, \bibinfo{person}{Gary Sivek}, {and} \bibinfo{person}{Ananda~Theertha Suresh}.} \bibinfo{year}{2019}\natexlab{}.
\newblock \showarticletitle{Agnostic federated learning}. In \bibinfo{booktitle}{\emph{Proceedings of the International Conference on Machine Learning}}. PMLR, \bibinfo{pages}{4615--4625}.
\newblock


\bibitem[Moro et~al\mbox{.}(2014)]%
        {moro2014data}
\bibfield{author}{\bibinfo{person}{S{\'e}rgio Moro}, \bibinfo{person}{Paulo Cortez}, {and} \bibinfo{person}{Paulo Rita}.} \bibinfo{year}{2014}\natexlab{}.
\newblock \showarticletitle{A data-driven approach to predict the success of bank telemarketing}.
\newblock \bibinfo{journal}{\emph{Decision Support Systems}}  \bibinfo{volume}{62} (\bibinfo{year}{2014}), \bibinfo{pages}{22--31}.
\newblock


\bibitem[Pan et~al\mbox{.}(2023)]%
        {pan2023fedmdfg}
\bibfield{author}{\bibinfo{person}{Zibin Pan}, \bibinfo{person}{Shuyi Wang}, \bibinfo{person}{Chi Li}, \bibinfo{person}{Haijin Wang}, \bibinfo{person}{Xiaoying Tang}, {and} \bibinfo{person}{Junhua Zhao}.} \bibinfo{year}{2023}\natexlab{}.
\newblock \showarticletitle{Fedmdfg: Federated learning with multi-gradient descent and fair guidance}. In \bibinfo{booktitle}{\emph{Proceedings of the AAAI Conference on Artificial Intelligence}}, Vol.~\bibinfo{volume}{37}. \bibinfo{pages}{9364--9371}.
\newblock


\bibitem[Papadaki et~al\mbox{.}(2022)]%
        {papadaki2022minimax}
\bibfield{author}{\bibinfo{person}{Afroditi Papadaki}, \bibinfo{person}{Natalia Martinez}, \bibinfo{person}{Martin Bertran}, \bibinfo{person}{Guillermo Sapiro}, {and} \bibinfo{person}{Miguel Rodrigues}.} \bibinfo{year}{2022}\natexlab{}.
\newblock \showarticletitle{Minimax demographic group fairness in federated learning}. In \bibinfo{booktitle}{\emph{Proceedings of the ACM Conference on Fairness, Accountability, and Transparency}}. \bibinfo{pages}{142--159}.
\newblock


\bibitem[Papadaki et~al\mbox{.}(2024)]%
        {papadaki2024federated}
\bibfield{author}{\bibinfo{person}{Afroditi Papadaki}, \bibinfo{person}{Natalia Martinez}, \bibinfo{person}{Martin Bertran}, \bibinfo{person}{Guillermo Sapiro}, {and} \bibinfo{person}{Miguel Rodrigues}.} \bibinfo{year}{2024}\natexlab{}.
\newblock \showarticletitle{Federated fairness without access to sensitive groups}.
\newblock \bibinfo{journal}{\emph{arXiv preprint arXiv:2402.14929}} (\bibinfo{year}{2024}).
\newblock


\bibitem[Roh et~al\mbox{.}(2020)]%
        {roh2020fairbatch}
\bibfield{author}{\bibinfo{person}{Yuji Roh}, \bibinfo{person}{Kangwook Lee}, \bibinfo{person}{Steven~Euijong Whang}, {and} \bibinfo{person}{Changho Suh}.} \bibinfo{year}{2020}\natexlab{}.
\newblock \showarticletitle{Fairbatch: Batch selection for model fairness}.
\newblock \bibinfo{journal}{\emph{arXiv preprint arXiv:2012.01696}} (\bibinfo{year}{2020}).
\newblock


\bibitem[Salazar et~al\mbox{.}(2023)]%
        {salazar2023fair}
\bibfield{author}{\bibinfo{person}{Teresa Salazar}, \bibinfo{person}{Miguel Fernandes}, \bibinfo{person}{Helder Ara{\'u}jo}, {and} \bibinfo{person}{Pedro~Henriques Abreu}.} \bibinfo{year}{2023}\natexlab{}.
\newblock \showarticletitle{Fair-fate: Fair federated learning with momentum}. In \bibinfo{booktitle}{\emph{Proceedings of the International Conference on Computational Science}}. Springer, \bibinfo{pages}{524--538}.
\newblock


\bibitem[Singhal et~al\mbox{.}(2021)]%
        {singhal2021federated}
\bibfield{author}{\bibinfo{person}{Karan Singhal}, \bibinfo{person}{Hakim Sidahmed}, \bibinfo{person}{Zachary Garrett}, \bibinfo{person}{Shanshan Wu}, \bibinfo{person}{John Rush}, {and} \bibinfo{person}{Sushant Prakash}.} \bibinfo{year}{2021}\natexlab{}.
\newblock \showarticletitle{Federated reconstruction: Partially local federated learning}. In \bibinfo{booktitle}{\emph{Proceedings of the Advances in Neural Information Processing Systems}}, Vol.~\bibinfo{volume}{34}. \bibinfo{pages}{11220--11232}.
\newblock


\bibitem[Tan et~al\mbox{.}(2022)]%
        {tan2022towards}
\bibfield{author}{\bibinfo{person}{Alysa~Ziying Tan}, \bibinfo{person}{Han Yu}, \bibinfo{person}{Lizhen Cui}, {and} \bibinfo{person}{Qiang Yang}.} \bibinfo{year}{2022}\natexlab{}.
\newblock \showarticletitle{Towards personalized federated learning}.
\newblock \bibinfo{journal}{\emph{IEEE Transactions on Neural Networks and Learning Systems}} (\bibinfo{year}{2022}).
\newblock


\bibitem[Wang et~al\mbox{.}(2021)]%
        {wang2021federated}
\bibfield{author}{\bibinfo{person}{Zheng Wang}, \bibinfo{person}{Xiaoliang Fan}, \bibinfo{person}{Jianzhong Qi}, \bibinfo{person}{Chenglu Wen}, \bibinfo{person}{Cheng Wang}, {and} \bibinfo{person}{Rongshan Yu}.} \bibinfo{year}{2021}\natexlab{}.
\newblock \showarticletitle{Federated learning with fair averaging}.
\newblock \bibinfo{journal}{\emph{arXiv preprint arXiv:2104.14937}} (\bibinfo{year}{2021}).
\newblock


\bibitem[Wu et~al\mbox{.}(2024b)]%
        {wu2024heterogeneity}
\bibfield{author}{\bibinfo{person}{Yebo Wu}, \bibinfo{person}{Li Li}, \bibinfo{person}{Chunlin Tian}, \bibinfo{person}{Tao Chang}, \bibinfo{person}{Chi Lin}, \bibinfo{person}{Cong Wang}, {and} \bibinfo{person}{Cheng-Zhong Xu}.} \bibinfo{year}{2024}\natexlab{b}.
\newblock \showarticletitle{Heterogeneity-aware memory efficient federated learning via progressive layer freezing}. In \bibinfo{booktitle}{\emph{Proceedings of the IEEE/ACM International Symposium on Quality of Service (IWQoS)}}. IEEE, \bibinfo{pages}{1--10}.
\newblock


\bibitem[Wu et~al\mbox{.}(2024c)]%
        {wu2024neulite}
\bibfield{author}{\bibinfo{person}{Yebo Wu}, \bibinfo{person}{Li Li}, \bibinfo{person}{Chunlin Tian}, \bibinfo{person}{Dubing Chen}, {and} \bibinfo{person}{Chengzhong Xu}.} \bibinfo{year}{2024}\natexlab{c}.
\newblock \showarticletitle{NeuLite: Memory-efficient federated learning via elastic progressive training}.
\newblock \bibinfo{journal}{\emph{arXiv preprint arXiv:2408.10826}} (\bibinfo{year}{2024}).
\newblock


\bibitem[Wu et~al\mbox{.}(2024a)]%
        {wu2024breaking}
\bibfield{author}{\bibinfo{person}{Yebo Wu}, \bibinfo{person}{Li Li}, \bibinfo{person}{Chunlin Tian}, {and} \bibinfo{person}{Chengzhong Xu}.} \bibinfo{year}{2024}\natexlab{a}.
\newblock \showarticletitle{Breaking the memory wall for heterogeneous federated learning with progressive training}.
\newblock \bibinfo{journal}{\emph{arXiv preprint arXiv:2404.13349}} (\bibinfo{year}{2024}).
\newblock


\bibitem[Yu et~al\mbox{.}(2020)]%
        {yu2020fairness}
\bibfield{author}{\bibinfo{person}{Han Yu}, \bibinfo{person}{Zelei Liu}, \bibinfo{person}{Yang Liu}, \bibinfo{person}{Tianjian Chen}, \bibinfo{person}{Mingshu Cong}, \bibinfo{person}{Xi Weng}, \bibinfo{person}{Dusit Niyato}, {and} \bibinfo{person}{Qiang Yang}.} \bibinfo{year}{2020}\natexlab{}.
\newblock \showarticletitle{A fairness-aware incentive scheme for federated learning}. In \bibinfo{booktitle}{\emph{Proceedings of the AAAI/ACM Conference on AI, Ethics, and Society}}. \bibinfo{pages}{393--399}.
\newblock


\bibitem[Yuan et~al\mbox{.}(2021)]%
        {yuan2021we}
\bibfield{author}{\bibinfo{person}{Honglin Yuan}, \bibinfo{person}{Warren~Richard Morningstar}, \bibinfo{person}{Lin Ning}, {and} \bibinfo{person}{Karan Singhal}.} \bibinfo{year}{2021}\natexlab{}.
\newblock \showarticletitle{What do we mean by generalization in federated learning?}. In \bibinfo{booktitle}{\emph{Proceedings of the International Conference on Learning Representations}}.
\newblock


\bibitem[Yue et~al\mbox{.}(2023)]%
        {yue2023gifair}
\bibfield{author}{\bibinfo{person}{Xubo Yue}, \bibinfo{person}{Maher Nouiehed}, {and} \bibinfo{person}{Raed Al~Kontar}.} \bibinfo{year}{2023}\natexlab{}.
\newblock \showarticletitle{Gifair-fl: A framework for group and individual fairness in federated learning}.
\newblock \bibinfo{journal}{\emph{INFORMS Journal on Data Science}} \bibinfo{volume}{2}, \bibinfo{number}{1} (\bibinfo{year}{2023}), \bibinfo{pages}{10--23}.
\newblock


\bibitem[Zeng et~al\mbox{.}(2021)]%
        {zeng2021improving}
\bibfield{author}{\bibinfo{person}{Yuchen Zeng}, \bibinfo{person}{Hongxu Chen}, {and} \bibinfo{person}{Kangwook Lee}.} \bibinfo{year}{2021}\natexlab{}.
\newblock \showarticletitle{Improving fairness via federated learning}.
\newblock \bibinfo{journal}{\emph{arXiv preprint arXiv:2110.15545}} (\bibinfo{year}{2021}).
\newblock


\bibitem[Zeng et~al\mbox{.}(2023)]%
        {zeng2023federated}
\bibfield{author}{\bibinfo{person}{Yuchen Zeng}, \bibinfo{person}{Hongxu Chen}, {and} \bibinfo{person}{Kangwook Lee}.} \bibinfo{year}{2023}\natexlab{}.
\newblock \showarticletitle{Federated learning with local fairness constraints}. In \bibinfo{booktitle}{\emph{Proceedings of the International Symposium on Information Theory}}. IEEE, \bibinfo{pages}{1937--1942}.
\newblock


\bibitem[Zhang et~al\mbox{.}(2018)]%
        {zhang2018mitigating}
\bibfield{author}{\bibinfo{person}{Brian~Hu Zhang}, \bibinfo{person}{Blake Lemoine}, {and} \bibinfo{person}{Margaret Mitchell}.} \bibinfo{year}{2018}\natexlab{}.
\newblock \showarticletitle{Mitigating unwanted biases with adversarial learning}. In \bibinfo{booktitle}{\emph{Proceedings of the AAAI/ACM Conference on AI, Ethics, and Society}}. \bibinfo{pages}{335--340}.
\newblock


\bibitem[Zhang et~al\mbox{.}(2020)]%
        {zhang2020fairfl}
\bibfield{author}{\bibinfo{person}{Daniel~Yue Zhang}, \bibinfo{person}{Ziyi Kou}, {and} \bibinfo{person}{Dong Wang}.} \bibinfo{year}{2020}\natexlab{}.
\newblock \showarticletitle{Fairfl: A fair federated learning approach to reducing demographic bias in privacy-sensitive classification models}. In \bibinfo{booktitle}{\emph{Proceedings of the IEEE International Conference on Big Data}}. IEEE, \bibinfo{pages}{1051--1060}.
\newblock


\end{thebibliography}

\clearpage
\appendix
\section{Appendix}
\subsection{Parameter Settings}
The batch size of preference vectors $n_{\lambda}$ is set to 64. The total communication rounds $T$ is set to 10, and the batch size of data is 128 in the neural network training. We use Adam as optimizer and the learning rate $\eta$ is 0.005. The smooth factor of smooth Tchebycheff function is set to 1. The reference point $\bm{r}$ in calculating hypervolume is set to (1, 1).
\subsection{Datasets}
According to the setting of \cite{zeng2021improving,ezzeldin2023fairfed}, the COMPAS and ADULT datasets have two clients, while the SYNTHETIC and BANK have three clients. In Section \ref{moreclient}, we will conduct experiments on more clients. We split 30\% of the data into the test set, and the other 70\% of the data are split into the training set and validation set with a ratio of 9:1. Furthermore, we select the model that performs best in the validation set for testing.

\subsection{Proofs}
\subsubsection{Proof of Lemma \ref{lemmaconvex}}
\begin{proof} Let $\widetilde{g}(\bm{\phi}_k)$ and $\mathcal{L}_i(\bm{\phi}_k)$ be simply the parameter $\bm{\phi}_k$ optimized by $\widetilde{g}$ and $\mathcal{L}_i$, respectively. Let $z \in [0,1]$, we start from the definition of strong convexity to obtain the strong convexity coefficient of $\widetilde{g}$.
\begin{align}
&\widetilde{g}_{t}(z\bm{\phi}_{k}^{(a)} + (1-z)\bm{\phi}_{k}^{(b)} \mid \lambda)  \\ &= \frac{1}{\gamma} \log\left(\sum_{i\in[3]} e^{\gamma(\mathcal{L}_i(z\bm{\phi}_{k}^{(a)} + (1-z)\bm{\phi}_{k}^{(b)}) ) / \lambda_i}\right) \\
&\leq \frac{1}{\gamma} \log\left(\sum_{i\in[3]} e^{\gamma(z\mathcal{L}_i(\bm{\phi}_{k}^{(a)}) + (1-z)\mathcal{L}_i(\bm{\phi}_{k}^{(b)}) ) / \lambda_i}\right) \\
&\quad - \log\left(\sum_{i\in[3]} e^{\frac{\mu_{i}}{2\lambda_{i}}}\right) \cdot z(1-z) ||\bm{\phi}_{k}^{(a)} - \bm{\phi}_{k}^{(b)}|| \\
&= \frac{z}{\gamma} \log\left(\sum_{i\in[3]} e^{\frac{\gamma\mathcal{L}_i(\bm{\phi}_{k,a})}{\lambda_i}}\right) + \frac{1-z}{\gamma} \log\left(\sum_{i\in[3]} e^{\frac{\gamma\mathcal{L}_i(\bm{\phi}_{k,b})}{\lambda_i}}\right) \\
&\quad - \log\left(\sum_{i\in[3]} e^{\frac{\mu_{i}}{2\lambda_{i}}}\right) \cdot z(1-z) ||\bm{\phi}_{k}^{(a)} - \bm{\phi}_{k}^{(b)}||.
\end{align}

\end{proof}

\subsubsection{Proof of Lemma \ref{lemmasmooth}}

\begin{proof}
    Given loss function $\mathcal{L}_{i}(\bm{x},y,\bm{\theta}(\bm{\lambda}))$ is $L_i$-smooth for $i=1,2$ and $\bm{\lambda} \in \Lambda$, we now prove that the smooth Tchebycheff function $\widetilde{g}(\bm{x},y,\bm{\theta}(\bm{\lambda})\mid \bm{\lambda})$ is $\frac{L_i}{\lambda}_i$-smooth for both the personalized model $\bm{\phi}_k$ and the communicated model $\bm{\psi}_k$ for $k=1,...,K$. We take $\bm{\phi}_k$ as an example, and let $\widetilde{g}(\bm{\phi}_k)$ and $\mathcal{L}_i(\bm{\phi}_k)$ be simply the parameter $\bm{\phi}_k$ optimized by $\widetilde{g}$ and $\mathcal{L}_i$, respectively. We start from the definition of smoothness to obtain the smooth coefficient of $\widetilde{g}$. Due to the chain rule $\nabla \widetilde{g}(\bm{\phi}_{k})= \sum_{i \in [2]}\nabla \mathcal{L}_{i}(\bm{\phi}_{k}) \nabla_{\mathcal{L}_i} \widetilde{g}(\bm{\phi}_{k})$, we can calculate $\nabla_{\mathcal{L}_i} \widetilde{g}(\bm{\phi}_{k})=\sum_{i\in[2]}\frac{e^{\mathcal{L}_{i}(\bm{\phi}_{k})}}{\lambda_{i}\sum_{i\in[2]}e^{\mathcal{L}_{i}(\bm{\phi}_{k})}}$. After that, we start from the following equation:
\begin{align}
        &\|\nabla \widetilde{g}(\bm{\phi}_{k}^{(a)})-\nabla \widetilde{g}(\bm{\phi}_{k}^{(b)}) \| \nonumber
        \\ 
        \vspace{-5cm}
        = & \| \sum_{i \in [2]} \! \nabla \mathcal{L}_{i}(\bm{\phi}_{k}^{(a)}) \nabla_{\mathcal{L}_i} \widetilde{g}(\bm{\phi}_{k}^{(a)})\! - \!\nabla \mathcal{L}_{i}(\bm{\phi}_{k}^{(b)}) \nabla_{\mathcal{L}_i} \widetilde{g}(\bm{\phi}_{k}^{(b)}) \! \|
\end{align}
\begin{align}
        = &\|\sum_{i\in[2]}\frac{e^{\mathcal{L}_{i}(\bm{\phi}_{k}^{(a)})}}{\lambda_{i}\sum_{i\in[2]}e^{\mathcal{L}_{i}(\bm{\phi}_{k}^{(a)})}}\nabla \mathcal{L}_{i}(\bm{\phi}_{k}^{(a)})\nonumber \\
        &-\sum_{i\in[2]}\frac{e^{\mathcal{L}_{i}(\bm{\phi}_{k}^{(b)})}}{\lambda_{i}\sum_{i\in[2]}e^{\mathcal{L}_{i}(\bm{\phi}_{k}^{(b)})}}\nabla \mathcal{L}_{i}(\bm{\phi}_{k}^{(b)})\| 
        \\
        \leq& \sum_{i\in[2]}\frac{1}{\lambda_{i}}  \|\nabla L_{i}(\bm{\phi}_{k}^{(a)})-\nabla L_{i}(\bm{\phi}_{k}^{(b)})\|\\ 
        \leq& \sum_{i\in[2]}\frac{L_{i}}{\lambda_{i}} \|\bm{\phi}_{k}^{(a)}-\bm{\phi}_{k}^{(b)}\|^{2}.
\end{align}

The last step is due to the smoothness of $\mathcal{L}_{i}$. Then, we can also get the smoothness coefficient of $\widetilde{g}(\bm{\phi})$,
\begin{equation}
    \|\nabla \widetilde{g}(\bm{\phi}^{(a)})-\nabla \widetilde{g}(\bm{\phi}^{(b)}) \| \leq \sum_{i\in[2]}\frac{L_{i}}{\lambda_{i}} \|\bm{\phi}^{(a)}-\bm{\phi}^{(b)}\|^{2}.
\end{equation}
\end{proof}
\subsubsection{Proof of Lemma \ref{exist}}
\begin{proof}
     Suppose that none of the optimal solution (i.e., $\bm{\phi}_{k}$) in problem (\ref{bii}) is Pareto optimal. Let ${\bm{\theta}^{*}_k}(\bm{\lambda})\in \Phi$ be an optimal solution to problem (\ref{bii}). Since we assume that it is not Pareto optimal, there must exist a model $\bm{\theta}^{'}_{k}(\bm{\lambda}) \in \Phi$ which is not optimal for problem (\ref{bii}) but for which $\mathcal{L}_i(\bm{x},y, \bm{\theta}^{'}_{k}(\bm{\lambda}))\leq \mathcal{L}_i(\bm{x},y,\bm{\theta}^{*}_{k}(\bm{\lambda}))$, $\forall i \in \{1, 2\}$ and $\mathcal{L}_i(\bm{x},y,\bm{\theta}^{'}_{k}(\bm{\lambda})) < \mathcal{L}_i(\bm{x},y,\bm{\theta}^{*}_{k}(\bm{\lambda})$), $\exists i \in \{1, 2\}$.
Therefore, we have
\begin{equation}
\!\! \mathcal{L}_i(\bm{x},y,\bm{\theta}^{'}_{k}(\bm{\lambda}))\leq \mathcal{L}_i(\bm{x},y,\bm{\theta}^{*}_{k}(\bm{\lambda})) , \forall i \in \{1, 2\}.
\end{equation}

We substitute this inequality into the smooth Tchebycheff function and get the following inequality
\begin{equation}\label{tch_inqu}
\frac{1}{\gamma} \text{log}\sum_{i\in[2]}e^{  \frac{\mathcal{L}_i(\bm{x},y,\bm{\theta}^{'}_{k}(\bm{\lambda}))}{\lambda_i}}  < \frac{1}{\gamma} \text{log}\sum_{i\in[2]}e^{  \frac{\mathcal{L}_i(\bm{x},y,\bm{\theta}^{*}_{k}(\bm{\lambda}))}{\lambda_i}}.
\end{equation}

This means $\bm{\theta}^{'}_{k}$ is optimal of problem (\ref{bii}), but we suppose $\bm{\theta}^{*}_{k}$ is optimal of problem (\ref{bii}). This contradiction completes the proof.
\end{proof}

\subsubsection{Proof of Lemma \ref{optima}}
\begin{proof}
We suppose that ${\bm{\theta}_{k}}^*(\bm{\lambda})$ is not Pareto optimal. By the definition of Pareto optimality (Definition \ref{paretoop}), there exists another model ${\bm{\theta}^{'}_{k}}(\bm{\lambda}) \in \Theta$ such that $\bm{\mathcal{L}}(\bm{x}, y, \bm{\theta}^{'}_{k}) \prec \bm{\mathcal{L}}(\bm{x}, y, \bm{\theta}^{*}_{k}) $. Therefore, we have the following inequality
\begin{equation}\label{dominate}
   \mathcal{L}_{i}(\bm{x},y, \bm{\theta}^{'}_{k}(\bm{\lambda})) \leq \mathcal{L}_{i}(\bm{x},y, \bm{\theta}^{*}_{k}(\bm{\lambda})), \ \forall i \in \{1, 2\}.
\end{equation}

We also can obtain inequality (\ref{tch_inqu}), which contradicts the assumption, so $\bm{\theta}_{k}^{*}(\bm{\lambda})$ is Pareto optimal for problem (\ref{bii}).

\end{proof}

\subsubsection{Proof of Theorem \ref{pf}}

\begin{proof}
Given a preference vector $\bm{\lambda} \triangleq ({\lambda_{i}>0}, i \in \{1, 2\})$, let $\bm{\theta}^{*}_{k}$ is defined as follows: 
\begin{equation}
\boldsymbol{\theta}_{k}^*(\bm{\lambda}) \triangleq \arg\min_{\boldsymbol{\theta}_{k}\in \Theta} \widetilde{g} (\bm{x},y, {\bm{\theta}_k}(\bm{\lambda}) \mid \bm{\lambda}).
\end{equation}

    We suppose that $\bm{\theta}_{k}^*(\bm{\lambda})$ is not Pareto optimal. By the definition of Pareto optimality (Definition 2.2), there exists another personalized model ${\bm{\theta}^{'}_{k}}(\bm{\lambda}) \in \Phi$ such that $\bm{\mathcal{L}}(\bm{x},y, \bm{\theta}^{'}_{k}(\bm{\lambda})) \prec \bm{\mathcal{L}}(\bm{x},y, \bm{\theta}^{*}_{k}(\bm{\lambda})) $. Therefore, we have the following inequality:
\begin{equation}\label{dominate}
   \mathcal{L}_{i}(\bm{x},y, \bm{\theta}^{'}_{k}(\bm{\lambda})) \leq \mathcal{L}_{i}(\bm{x},y, \bm{\theta}^{*}_{k}(\bm{\lambda})), \ \forall i \in \{1, 2\}.
\end{equation}

We also can obtain inequality (\ref{tch_inqu}), which contradicts the assumption, so $\bm{\theta}_{k}^{*}$ is Pareto optimal for problem (\ref{bii}).
\end{proof}

\subsubsection{Proof of Proposition \ref{relation}}
\begin{proof}
    We now prove an upper bound of the difference between the personalized model on any client (take client $k$ as an example) $\bm{\phi}_k$ and the optimal personalized model $\bm{\phi}_k^{*}$. Let $I^{t}_{k}$ indicate if user $k$ is selected at the $t$-round, and $\mathbb{E}\left[I^{t}_{k}\right]=p_k$. We have 
\begin{align}
\mathbb{E}&\left[\|\bm{\phi}_k^{t+1}-\bm{\phi}_k^*\|^2\right]& \nonumber \\= & \ \mathbb{E}\left[\|\bm{\phi}_k^t-\eta I^{t}_{k}\nabla \widetilde{g}(\bm{\phi}_k^t;\bm{\psi}^t)-\bm{\phi}_k^*\|^2\right] \\
=& \ \mathbb{E}\left[\|\bm{\phi}_k^t-\bm{\phi}_k^*\|^2\right]+\eta^2\mathbb{E}[\left\| I^{t}_{k} \widetilde{g}(\bm{\phi}_k^t;\bm{\psi}^t)\|^2\right]\nonumber\\
&+2\eta\mathbb{E}\left[p_{k}{\nabla \widetilde{g}(\bm{\phi}_k^t;\bm{\psi}^t)}^{T}(\bm{\phi}_k^*-\bm{\phi}_k^t)\right] \\
\overset{(a)}{\leq}&(1-\eta p_k\check{\mu})\mathbb{E}\left[\|\bm{\phi}_k^t-\bm{\phi}_k^*\|^2\right]\nonumber+\underbrace{\eta^2\mathbb{E}\left[\|\nabla \widetilde{g}(\bm{\phi}_k^t;\bm{\psi}^t)\|^2\right]}_{\mathcal{C}_{1}} \nonumber\\
&+\underbrace{2\eta p_k\mathbb{E}\left[\widetilde{g}(\bm{\phi}_k^*;\bm{\psi}^t)-\widetilde{g}(\bm{\phi}_k^t;\bm{\psi}^t)\right]}_{\mathcal{C}_{2}},\label{total}
\end{align}
where $(a)$ is due to the smoothness of $\widetilde{g}$. We then find the upper bounds of $\mathcal{C}_1$ and $\mathcal{C}_2$, respectively. For $\mathcal{C}_1$,
\begin{align} 
\eta^{2}&\mathbb{E}\left[\|\nabla \widetilde{g}(\bm{\phi}_{k}^{t};\bm{\psi}^t)\|^{2}\right] \nonumber\\ =&\eta^{2}\mathbb{E}\left[\|\nabla \widetilde{g}(\bm{\phi}_{k}^{t};\bm{\psi}^t)-\nabla \widetilde{g}(\bm{\phi}_{k}^{t};\bm{\psi}^*) +\nabla \widetilde{g}(\bm{\phi}_{k}^{t};\bm{\psi}^*)\|^{2}\right]\\
=&2\eta^{2}\mathbb{E}\left[\|\nabla \widetilde{g}(\bm{\phi}_{k}^{t};\bm{\psi}^*)\|  \|\nabla \widetilde{g}(\bm{\phi}_{k}^{t};\bm{\psi}^t)-\nabla \widetilde{g}(\bm{\phi}_{k}^{t};\bm{\psi}^*)\|\right] \nonumber\\
&\!+\!\eta^{2}\mathbb{E}\left[\|\nabla \widetilde{g}(\bm{\phi}_{k}^{t};\bm{\psi}^*)\|^{2}\right]\!\!+\!\eta^{2}\mathbb{E}\left[\|\nabla \widetilde{g}(\bm{\phi}_{k}^{t};\bm{\psi}^t)\!-\!\!\nabla \widetilde{g}(\bm{\phi}_{k}^{t};\bm{\psi}^*)\|^{2}\right]\\
\overset{(b)}{\leq}&\eta^{2}\mathbb{E}\left[\|\nabla \widetilde{g}(\bm{\phi}_{k}^{t};\bm{\psi}^*)\|^{2}\right]\!\!+\!2\eta^{2} \check{L}\mathbb{E}\left[\|\nabla \widetilde{g}(\bm{\phi}_{k}^{t};\bm{\psi}^*)\| \|\bm{\psi}^t-\bm{\psi}^*\|\right] \nonumber\\
&+\eta^{2}\check{L}^2\mathbb{E}\left[\|\bm{\psi}^t-\bm{\psi}^*\|^{2}\right],\label{up1}
\end{align}
where $(b)$ is due to the convexity of $\widetilde{g}$. Then, we continue to find the upper bound of $\mathcal{C}_{2}$,
\begin{align}
2\eta &p_k\mathbb{E}\left[\widetilde{g}(\bm{\phi}_k^*;\bm{\psi}^t)-\widetilde{g}(\bm{\phi}_k^t;\bm{\psi}^t)\right] \nonumber\\ \overset{(c)}{\leq}2&\eta p_k\mathbb{E}\left[\widetilde{g}(\bm{\phi}_k^*;\bm{\psi}^t)-\widetilde{g}(\bm{\phi}_k^t;\bm{\psi}^t) + \eta p_k(\check{L}-\check{\mu})\|\bm{\psi}^t-\bm{\psi}^*\|^2\right]\nonumber\\
+&\mathbb{E}\left[(\nabla \widetilde{g}(\bm{\phi}_k^*;\bm{\psi}^t)-\nabla \widetilde{g}(\bm{\phi}_k^t;\bm{\psi}^t))^{\mathsf{T}}(\bm{\psi}^t-\bm{\psi}^*)\right] 
\\
\overset{(d)}{\leq}2&\eta p_k\mathbb{E}\left[\widetilde{g}(\bm{\phi}_k^*;\bm{\psi}^t)-\widetilde{g}(\bm{\phi}_k^t;\bm{\psi}^t) + \eta p_k(\check{L}-\check{\mu})\|\bm{\psi}^t-\bm{\psi}^*\|^2\right]\nonumber\\
+ &2\eta p_{k}\check{L}\mathbb{E}\left[\|\bm{\phi}_k^t-\bm{\phi}_k^*\| \|\bm{\psi}^t-\bm{\psi}^*\|\right],\label{up2} 
\end{align}
where $(c)$ uses the smoothness and convexity of $\widetilde{g}$. $(d)$ uses the convexity of $\widetilde{g}$. So far, we have found the upper bounds of $\mathcal{C}_{1}$ and $\mathcal{C}_{2}$. We then plugging their upper bounds Eq. (\ref{up1}) and Eq. (\ref{up2}) into Eq. (\ref{total}), as follows:
\begin{align}
\mathbb{E}&\left[\|\bm{\phi}_k^{t+1}-\bm{\phi}_k^*\|^2\right]\nonumber\\ \leq
&(1-\eta p_k \check{\mu})\mathbb{E}\left[\|\bm{\phi}_k^t-\bm{\phi}_k^*\|^2\right]+
\eta^{2}\mathbb{E}\left[\|\nabla \widetilde{g}(\bm{\phi}_{k}^{t};\bm{\psi}^*)\|^{2}\right]\nonumber\\
&+2\eta^{2} \check{L}\mathbb{E}\left[\|\nabla \widetilde{g}(\bm{\phi}_{k}^{t};\bm{\psi}^*)\| \|\bm{\psi}^t-\bm{\psi}^*\|\right] \nonumber\\
&+\eta^{2}\check{L}^2\mathbb{E}\left[\|\bm{\psi}^t-\bm{\psi}^*\|^{2}\right]+2\eta p_k\mathbb{E}\big[\widetilde{g}(\bm{\phi}_k^*;\bm{\psi}^t)-\widetilde{g}(\bm{\phi}_k^t;\bm{\psi}^t) \nonumber\\
&+2\eta p_{k}\check{L} \mathbb{E} \|\bm{\phi}_k^t-\bm{\phi}_k^*\| \|\bm{\psi}^t-\bm{\psi}^*\|\big]\nonumber\\
&+ \eta p_k(\check{L}- \check{\mu})\mathbb{E}\left[\|\bm{\psi}^t\!-\!\bm{\psi}^*\|^2\right]  \\ \overset{(e)}{\leq}
&(1-\eta p_k \check{\mu})\mathbb{E}\left[\|\bm{\phi}_k^t-\bm{\phi}_k^*\|^2\right]+
\eta^{2}\mathbb{E}\left[\|\nabla \widetilde{g}(\bm{\phi}_{k}^{t};\bm{\psi}^*)\|^{2}\right]\nonumber\\
&+2\eta^{2} \check{L}\sqrt{\mathbb{E}\left[\|\nabla \widetilde{g}(\bm{\phi}_{k}^{t};\bm{\psi}^*)\right\|^2 \mathbb{E}\left[\|\bm{\psi}^t\!-\! \bm{\psi}^*\|^2\right]} \nonumber\\
&+\eta^{2}\check{L}^2\mathbb{E}\left[\|\bm{\psi}^t-\bm{\psi}^*\|^{2}\right] +2\eta p_{k}\check{L}\sqrt{\mathbb{E}\|\bm{\phi}_k^t-\bm{\phi}_k^*\|^2 \mathbb{E} \|\bm{\psi}^t-\bm{\psi}^*\|^2}\nonumber\\
&+ \eta p_k(\check{L}- \check{\mu})\mathbb{E}\left[\|\bm{\psi}^t\!-\!\bm{\psi}^*\|^2\right]  \\
\overset{(f)}{\leq} &(1\!-\!\eta p_k\check{\mu})\mathbb{E}\left[\|\bm{\phi}_k^t\!-\!\bm{\phi}_k^*\|^2\right] \!+\! \eta^{2}G^{2}\!+\!2\eta^2 G\check{L} \sqrt{\mathbb{E}[\|\bm{\psi}^t-\bm{\psi}^*\|^{2}}] \nonumber\\ &+(\eta^{2}\check{L}^2+ \eta p_k(\check{L}-\check{\mu}))\mathbb{E}\left[\|\bm{\psi}^t-\bm{\psi}^*\|^{2}\right] \nonumber\\
&+2\eta p_k\check{L}\sqrt{\mathbb{E}\left[\|\bm{\phi}_k^t-\bm{\phi}_k^*\|^2\right] \mathbb{E} \left[\|\bm{\psi}^t-\bm{\psi}^*\|^{2}\right]}, 
\end{align}
where $(e)$ is due to the Cauchy-Schwarz inequality and $\mathbb{E}\big[\widetilde{g}(\bm{\phi}_k^*;\bm{\psi}^t)-\widetilde{g}(\bm{\phi}_k^t;\bm{\psi}^t)\big]\leq 0$. $(f)$ holds because the norm of the squared gradient is bounded by $G^2$ (Assumption \ref{bound_g}).
\end{proof}

\subsubsection{Proof of Theorem \ref{rela}}
\begin{proof}
There exists $C<\infty$ such that for any client $k\in [K]$, $C>\frac{\mathbb{E}\left[\|\bm{\phi}_k^0-\bm{\phi}_k^*\|^2\right]}{z(0)}$, we have
\begin{equation}
     \mathbb{E}\left[\|\bm{\phi}_k^0-\bm{\phi}_k^*\|^2\right]\leq Cz(0).
\end{equation}  
If $\mathbb{E}\left[\|\bm{\phi}_k^t-\bm{\phi}_k^*\|^2\right]\leq Cz(t)$ holds, then for $t+1$,
    \begin{align}
\mathbb{E}&\left[||\bm{\phi}_k^{t+1}-\bm{\phi}_k^*||^2\right] \nonumber\\ \leq&(1-\frac{2z(t)}{A})Cz(t)+\frac{z(t)^{2}}{A}\frac{4\sqrt{C}}{\check{\mu}} + \frac{z(t)^{2}}{A} \frac{2(\check{L}-\check{\mu})}{\check{\mu}} \nonumber\\ &+\frac{z(t)^{2}}{A}\frac{4}{Ap_{k}^{2} \check{\mu}^2}(G^{2}+\check{L}^2z(t)+2G\check{L} \sqrt{z(t)})\\
\overset{(a)}{\leq}&(1-\frac{2z(t)}A)Cz(t)+\frac{Cz(t)^2}A\\
=&(1-\frac{z(t)}{A})Cz(t) \\
\overset{(b)}{\leq} &Cz(t+1) \\
\overset{(c)}{\leq}& C \cdot O(\log(\frac{1}{t+1})),
\label{perconv}
\end{align}
where $(a)$ holds because $C$ is a sufficiently large number such that coefficient of $\frac{z(t)^2}{A}$ is smaller that $C$. $(b)$ is due to the condition $\frac{z(t+1)}{z(t)}\geq1-\frac{z(t)}{A}$. $(c)$ is due to the Lemma \ref{commconv}.  
\end{proof}


\end{document}